\documentclass{article}
\usepackage[margin=1in]{geometry}
\usepackage{mathtools}
\usepackage{amssymb}
\usepackage{amsthm}
\usepackage[round]{natbib}
\usepackage{bm,mathrsfs}
\usepackage{soul}
\usepackage{microtype}
\usepackage{xcolor}
\usepackage{graphicx}
\usepackage{enumitem,dsfont}
\usepackage[colorlinks=true,linkcolor=blue,citecolor=blue,urlcolor=blue]{hyperref}

\theoremstyle{plain}
\newtheorem{theorem}{Theorem}[section]
\newtheorem{proposition}{Proposition}[section]
\newtheorem{lemma}{Lemma}[section]
\newtheorem{corollary}{Corollary}[section]
\newtheorem{assumption}{Assumption}[section]

\theoremstyle{definition}

\theoremstyle{remark}
\newtheorem*{remark}{Remark}

\NewDocumentCommand{\eqs}{s m}{
  \IfBooleanTF{#1}{
    \begin{equation*}
    \begin{aligned}
    #2
    \end{aligned}
    \end{equation*}
  }{
    \begin{equation}
    \begin{aligned}
    #2
    \end{aligned}
    \end{equation}
  }
}

\DeclareMathOperator*{\argmin}{arg\,min}

\newcommand{\Cov}{\operatorname{Cov}}

\newcommand{\Var}{\operatorname{Var}}

\newcommand{\1}{\mathbf{1}}
\newcommand{\R}{\mathbb{R}}
\newcommand{\E}{\mathbb{E}}
\newcommand{\F}{\mathcal{F}}

\newcommand{\Prob}{\mathbb{P}}

\newcommand{\lin}{\mathrm{lin}}

\newcommand{\N}{\mathcal{N}}

\newcommand{\deq}{\stackrel{\mathrm{d}}{=}}

\newcommand{\pto}{\stackrel{\mathrm{p}}{\to}}
\newcommand{\longpto}{\xrightarrow{\mathrm{p}}}
\newcommand{\asto}{\stackrel{\mathrm{a.s.}}{\to}}

\newcommand{\Y}{\mathcal{Y}}

\newcommand{\calL}{\mathcal{L}}

\begin{document}

\title{Limit Theory for Online Expectile Estimation under Short and Long Range Dependence}
\author{
    Younghun Lee\textsuperscript{1}, Jake Zummo\textsuperscript{1}, and Wei Biao Wu\textsuperscript{1}\\[0.5em]
    \textit{\textsuperscript{1}Department of Statistics, University of Chicago}
}
\date{October 11, 2026}

\maketitle

\begin{abstract}
Expectiles are tail-sensitive distributional summaries with important applications in financial risk measurement and forecasting, yet statistical theory for their online estimation from dependent data remains limited. We study a one-pass, stochastic-gradient algorithm with constant learning rate $\gamma$ for estimating the marginal expectile of a stationary causal process. The recursion admits a unique stationary causal solution and geometrically forgets its initialization. Under short-range dependence, its last iterate is asymptotically Gaussian at the $\sqrt{\gamma}$ scale, with variance governed by the population curvature and the long-run variance of the expectile score. For long-memory linear processes the scale changes to $\gamma^{\beta-1/2}\ell(1/\gamma)$, with a Gaussian limit arising from a fractional Brownian motion integral. We further derive the joint long-memory limit for multiple expectile levels and characterize its limiting covariance structure. Various numerical experiments corroborate the predicted Gaussian approximations and regime-specific fluctuation scales.
\end{abstract}

\section{Introduction}

Expectiles are an asymmetrically weighted analogue of the mean \citep{Newey1987Asymmetric,ren2026asymptotics}.
The $\tau$-expectile of a univariate distribution minimizes the expected asymmetric squared loss centered at the candidate expectile: squared deviations above and below it receive weights $\tau$ and $1-\tau$, respectively.
Because this weighting gives greater influence to deviations on one side of the distribution, expectiles are used in risk measurement and tail forecasting \citep{Taylor2008Estimating,kuan2009assessing,Bellini2014Generalized,ren2022financial,giacometti2021tail}.
The asymptotic theory for batch (empirical) expectile estimators, an instance of $M$-estimation, includes weak convergence of expectile processes and inference for extreme expectiles, that is, expectiles at levels near $0$ or $1$ \citep{Holzmann2016Expectile,Zwingmann2020Weak,Daouia2020Tail}. 
More recently, \citet{ren2026asymptotics} developed asymptotic theory for batch expectile estimators under short- and long-range dependence.
In many dependent-data applications, observations arrive sequentially, motivating the problem of one-pass, online estimation of the expectile of a stationary process.

We study the small-step distribution of the unaveraged last iterate of constant-step stochastic gradient descent (SGD) for expectile estimation under both short- and long-memory dependence in the data. In the limits below, the learning rate $\gamma\downarrow0$ while the iteration $n=n(\gamma)\to\infty$ sufficiently fast to remove the initialization effect (that is, after burn-in).

We quantify observation dependence through the Bernoulli-shift framework of \citet{wu2005nonlinear}. Score partial sums have order $n^{1/2}$ under short memory and the larger order $n^{3/2-\beta}\ell(n)$ for the linear long-memory inputs considered here \citep{Ho1997Limit,Wu2003Empirical}. The central question of this paper is how these memory properties propagate through the constant-step expectile recursion and determine the limiting distribution of its last iterate.

\paragraph{Related work.}
\emph{Constant-step stochastic approximation.}
Asymptotic normality of stochastic approximation under diminishing learning rates dates to \citet{Chung1954On}, while the stationary-distribution viewpoint for constant learning rates originates with \citet{pflug1986stochastic}. Modern analyses characterize convergence to the invariant measure, bias, and concentration under smooth stochastic gradients \citep{dieuleveut2020bridging,Chen2022Stationary,Merad2025Convergence}, and \citet{Goldreich2025Asymptotic} place the constant-step regime within a broader theory of learning-rate schedules. Building on iterated random functions and geometric moment contraction \citep{Wu2000A,Wu2004Limit}, \citet{li2026stochastic} obtain coupling and quenched limit theory for averaged constant-step SGD. \citet{Zhang2026Sharp} derive moment and concentration bounds for decaying-step SGD. In both, the incoming observations are independent and the temporal dependence is generated by the recursion itself. \citet{wei2025gaussian} analyze the nonstationary last iterate directly under joint limits in the iteration and the learning rate. Their setting uses independent observations, whereas we treat serially dependent Bernoulli-shift data.

\emph{Dependent data.}
Distributional results for stochastic approximation under dependent sampling have focused on Markovian noise, linear recursions, or averaged estimators \citep{Li2023Online,Huo2026Bias,Samsonov2025Statistical}, with related analyses under $\phi$-mixing and more general dependence \citep{Liu2023Statistical,Shen2026SGD}. Martingale approximations and invariance principles in the Bernoulli-shift framework accommodate nonlinear short-memory inputs \citep{wu2005nonlinear,Wu2011Asymptotic,Wu2004Martingale,Wu2007Strong}, while long memory changes normalizations and limits for empirical processes and $M$-estimators \citep{Koul1997Asymptotic,Hsing2004On}.

\paragraph{Our contributions.}
We establish marginal limits for the last iterate of constant-step expectile SGD and, under long memory, joint limits across expectile levels.
\begin{enumerate}
    \item Under short memory, we establish limit theory for the last iterate of constant-step expectile SGD driven by nonlinear Bernoulli-shift observations. Its limiting variance is $\Omega_g/(2A)$: serial dependence replaces the marginal variance of the score by its long-run variance $\Omega_g$, while $A$ is the population curvature.
    \item Under long memory, we establish marginal limit theory for constant-step expectile SGD driven by a linear process. This limit identifies the change in fluctuation scale, from $\sqrt{\gamma}$ under short memory to $\gamma^{\beta-1/2}\ell(1/\gamma)$ under long memory.
    \item Under long memory, we characterize the joint covariance across expectile levels. It has full rank when the population curvatures are distinct, whereas the batch limit has rank one. Centered cross-level contrasts can retain first-order fluctuations under SGD that cancel in batch estimation.
\end{enumerate}

\paragraph{Notation.}
For a random variable $Z$ and $p\ge1$ we write $\|Z\|_p:=(\E|Z|^p)^{1/p}$, and $\1(A)$ for the indicator function of the event $A$. 
For a function $h$, we write $\|h\|_{\infty}:=\sup _{x \in \mathbb{R}}|h(x)|$.
Convergence in distribution is denoted $\Rightarrow$, convergence in probability $\pto$, and equality in distribution $\deq$. Symbols $o_P$ and $O_P$ are used in the usual sense. The relation $a\lesssim b$ means $a\le Cb$ for a constant $C$ that does not depend on the index, the learning rate, or the horizon, and $a\asymp b$ means $a\lesssim b\lesssim a$; such a constant is written $C$ and may change from one line to the next.

\section{Expectile SGD under Dependence}\label{sec:overview}

We consider a streaming-data setting in which observations
$X_1,X_2,\ldots$ from a stationary Bernoulli-shift process arrive sequentially. Our goal is to estimate the marginal $\tau$-expectile using the last iterate of a one-pass, constant-step stochastic-gradient recursion. Throughout the paper, $(X_t)_{t\in\mathbb Z}$ admits the Bernoulli-shift representation
\begin{equation}\label{eq:wold-representation}
X_t=H(\varepsilon_t,\varepsilon_{t-1},\ldots),
\end{equation}
where $(\varepsilon_t)_{t\in\mathbb Z}$ are i.i.d.\ innovations and $H: \mathbb{R}^\mathbb{N} \rightarrow \mathbb R$ is
Borel measurable. We write $\F_t:=\sigma(\varepsilon_s:s\leq t)$ for the natural
innovation filtration, so that $X_t$ and $\theta_t$ are $\F_t$-measurable.
This representation describes a class of causal, ergodic,
stationary processes, with widely varying strengths of temporal dependence
\citep{wu2005nonlinear}.

To quantify dependence, let $\varepsilon_0'$ be an independent copy of $\varepsilon_0$, independent of the entire innovation sequence, and define
\[
  X_t^{\{0\}}:=H(\varepsilon_t,\ldots,\varepsilon_1,\varepsilon_0',\varepsilon_{-1},\ldots),
  \qquad t\ge0.
\]
For $p\ge1$ such that $X_0\in L^p$ (so the required moment exists), the lag-$t$ physical dependence coefficient is
\[
\delta_{X,p}(t):=\|X_t-X_t^{\{0\}}\|_p,
\qquad t\ge0.
\]
It measures the effect at time $t$ of replacing the innovation at time zero by an independent copy while keeping all other innovations fixed. Section~\ref{sec:short_memory} imposes summability of the associated physical dependence coefficients to obtain Gaussian approximation for short-memory processes, whereas Section~\ref{sec:long_memory} specializes the representation to a linear process whose coefficients are not summable. For a linear process, replacing $\varepsilon_0$ changes $X_t$ by $c_t(\varepsilon_0-\varepsilon_0')$, so its physical dependence coefficients equal $|c_t|\|\varepsilon_0-\varepsilon_0'\|_p$.

\subsection{Expectile target and SGD recursion}\label{subsec:target}

Fix an expectile level $\tau\in(0,1)$. For
$u=\theta-x$, define the univariate expectile loss, also known as the
asymmetric least-squares criterion \citep{Newey1987Asymmetric}, by
\begin{equation}\label{eq:expectile-loss}
f_\tau(u)
:=\frac12\bigl(\tau\1(u\leq0)
 +(1-\tau)\1(u>0)\bigr)u^2.
\end{equation}
The $\tau$-expectile of $X_0$ is the value of $\theta$ that minimizes the
expected loss, i.e.
\[
\theta_\tau^*
:=\argmin_{\theta\in\R}\E[f_\tau(\theta-X_0)].
\]
Given a constant learning rate $\gamma>0$ and an initialization $\theta_0$,
the expectile SGD recursion is
\begin{equation}\label{eq:expectile-sgd}
\theta_t=\theta_{t-1}-\gamma\,g_\tau(\theta_{t-1}-X_t),
\qquad t\geq1,
\end{equation}
where $g_\tau(u):=f_\tau'(u)=\bigl(\tau\1(u\leq0)+(1-\tau)\1(u>0)\bigr)u$ is the derivative of the expectile loss, also called the expectile score.
Unless several expectile levels are considered explicitly, we fix $\tau$ and write $g$ rather than $g_\tau$ for brevity. The \emph{score function} $g$ is piecewise linear with a kink at the origin when $\tau\ne1/2$ (and is linear when $\tau=1/2$). Write
$\underline\tau:=\min(\tau,1-\tau)$ and
$\overline\tau:=\max(\tau,1-\tau)$. For $u \neq 0$, we have
\begin{equation}\label{eq:score-derivative}
g'(u)=\tau\1(u<0)+(1-\tau)\1(u>0)
\in[\underline\tau,\overline\tau],
\end{equation}
and every slope of a secant of the graph of $g$ lies within $[\underline\tau,\overline\tau]$ as well. Hence $g$ is
$\overline\tau$-Lipschitz and strictly increasing. These bounds suffice for the contraction argument below; the later linearizations also use the piecewise-linearity of $g$, including control of crossings at its kink. From \eqref{eq:score-derivative}, we see that the loss $f_\tau$ is $\underline\tau$-strongly convex. 
Assuming $\E[X_0^2]<\infty$, the population risk $\theta\mapsto\E[f_\tau(\theta-X_0)]$ is finite and $\underline\tau$-strongly convex. 
Consequently, $\theta_\tau^*$ is its unique minimizer and is equivalently the unique solution of
$\E[g(\theta_\tau^*-X_0)]=0$ (the score equation). The differentiation argument is given in Appendix~\ref{app:population-score}.

\subsection{Recentering and limiting parameters}\label{subsec:recentering}

It is convenient to recenter both the observations and the iterates around the
expectile. Throughout the paper, we let $U_t:=X_t-\theta_\tau^*$ denote the
recentered observation and $e_t:=\theta_t-\theta_\tau^*$ the estimation
error. The first-order condition becomes $\E[g(-U_t)]=0$, and the
recursion \eqref{eq:expectile-sgd} can be written as
\begin{equation}\label{eq:centered-recursion}
e_t=e_{t-1}-\gamma\,g(e_{t-1}-U_t).
\end{equation}

In the short-memory regime, defined by the summability condition in
Assumption~\ref{ass:short-memory}, two population quantities determine the limiting behavior of the
recursion. The first is the derivative of the population score at the target,
\begin{equation}\label{eq:A-definition}
A:=\E[g'(-U_0)]
=\left.\frac{d}{d\theta}\E[g(\theta-X_0)]
  \right|_{\theta=\theta_\tau^*},
\end{equation}
which is the curvature of the population risk at $\theta_\tau^*$. The
second is the long-run variance of the score process,
\begin{equation}\label{eq:score-lrv}
\Omega_g:=\sum_{h\in\mathbb Z}\Cov(g(-U_0),g(-U_h)),
\end{equation}
through which the serial dependence of the observations enters. Conditions
under which $A$ is well defined and the series in \eqref{eq:score-lrv}
converges absolutely are imposed in Section~\ref{sec:short_memory}.

\subsection{Existence of a stationary solution and geometric forgetting}\label{sec:stability}

The expectile update map is uniformly contractive. This implies the existence of a unique stationary solution of the recursion. The results of this subsection apply to both short- and long-memory processes. For existence and contraction, no dependence-decay or mixing condition is imposed: stationarity of the input and $X_0\in L^p$ suffice.

For $x\in\R$, write recursion \eqref{eq:expectile-sgd} as $\theta_t=F_{X_t}(\theta_{t-1})$, where
\begin{equation}\label{eq:update-map}
\begin{aligned}
F_x(\theta)&:=\theta-\gamma g(\theta-x)\\
&=\begin{cases}
\theta-\gamma\tau(\theta-x), & \theta\le x,\\
\theta-\gamma(1-\tau)(\theta-x), & \theta>x.
\end{cases}
\end{aligned}
\end{equation}
The two pieces defining $F_x$ coincide at $\theta=x$, so $F_x$ is everywhere continuous. We impose the following step-size assumption.
\begin{assumption}[Step size]\label{ass:step-size}
The learning rate satisfies
\begin{equation}\label{eq:admissible-step}
0<\gamma<\overline\tau^{-1}.
\end{equation}
\end{assumption}
Under Assumption~\ref{ass:step-size}, piecewise linearity gives
\begin{equation*}
    |F_x(\theta) - F_x(\theta')| \leq (1 - \gamma \underline \tau)|\theta - \theta'|,
\end{equation*}
so $F_x$ is $(1 - \gamma \underline \tau)$-Lipschitz, uniformly in $x$.
Since $1-\gamma\underline\tau\in(0,1)$, $F_x$ is a contraction.
To construct the stationary distribution, define the $m$-step backward map at time $0$ as
\[
B_m(u):=F_{X_0}\circ F_{X_{-1}}\circ\cdots\circ F_{X_{-m+1}}(u),
\]
following the iterated-random-function construction; see \citet{DiaconisFreedman1999} and \citet{Wu2004Limit}.

\begin{proposition}[Stationary solution and geometric forgetting]\label{prop:stationary-solution}\label{prop:gmc}
Fix $\tau\in(0,1)$, $p\ge1$, and $\gamma$ satisfying Assumption~\ref{ass:step-size}. Suppose $(X_t)_{t\in\mathbb Z}$ is stationary and $X_0\in L^p$.
For every $u\in\R$, the backward iterates $B_m(u)$ converge almost surely and in $L^p$ as $m\to\infty$ to the same limit $\theta_0^\circ$. For each $t$, define $\theta_t^\circ$ analogously as the limit of the $m$-step backward map $F_{X_t}\circ\cdots\circ F_{X_{t-m+1}}(u)$ as $m\to\infty$. These limits form the unique stationary causal $L^p$ solution $(\theta_t^\circ)_{t\in\mathbb Z}$ of \eqref{eq:expectile-sgd}, i.e., they satisfy
\begin{equation}\label{eq:stationary-sgd-recursion}
\theta_t^\circ=F_{X_t}\big(\theta_{t-1}^\circ\big),\qquad t\in\mathbb Z.
\end{equation}
If a process $(\theta_t)_{t\geq0}$ follows recursion \eqref{eq:expectile-sgd}, driven
by the same observations $(X_t)_{t\geq1}$ with a possibly random initialization $\theta_0\in L^p$, then
\begin{equation}\label{eq:GMC}
\|\theta_t-\theta_t^\circ\|_p
\leq(1-\gamma\underline\tau)^t
\|\theta_0-\theta_0^\circ\|_p.
\end{equation}
Consequently, writing $\pi_\gamma:=\calL(\theta_0^\circ)$, we have
$\theta_t\Rightarrow\pi_\gamma$ as $t\to\infty$.
\end{proposition}

The class of Bernoulli-shift processes is closed under measurable transformations that use only current and past inputs. Thus, under the Bernoulli-shift assumption \eqref{eq:wold-representation} on $X$, the stationary SGD iterates also admit a Bernoulli-shift representation: there exists a measurable function $H$ such that
\[
\theta_t^\circ=H(\varepsilon_t,\varepsilon_{t-1},\ldots).
\]
Here $H$ depends on $\gamma$ and $\tau$, although these arguments are suppressed in its notation. Thus the stationary solution is itself a stationary Bernoulli-shift process driven by the same innovations. The contraction bound \eqref{eq:GMC} transfers results for the stationary solution to forward iterates after an explicit burn-in.

\section{Short-Memory Processes}\label{sec:short_memory}

We impose summability of the physical dependence coefficients introduced in Section~\ref{sec:overview} and record what it implies for the score process (Section~\ref{subsec:dependence}). We then localize the stationary solution at the $\sqrt\gamma$ scale and establish the central limit theorem for the actual last iterate (Section~\ref{sec:short-memory-limit}).

\subsection{Short-memory condition and localization}\label{subsec:dependence}

The short-memory condition is summability of the physical dependence coefficients defined in Section~\ref{sec:overview}.

\begin{assumption}[Short-range dependence]\label{ass:short-memory}
The process $(X_t)_{t\in\mathbb Z}$ admits the Bernoulli-shift representation \eqref{eq:wold-representation}, satisfies $X_0\in L^2$, and
\begin{equation}\label{eq:short-memory-fdm}
\Theta_{X,2}:=\sum_{t=0}^\infty\delta_{X,2}(t)<\infty .
\end{equation}
\end{assumption}

The summability condition controls the cumulative effect of a single innovation and is the short-memory assumption used to obtain covariance summability, which makes the long-run variance $\Omega_g$ in \eqref{eq:score-lrv} well defined, and martingale approximations.
Specifically, let $U_t^{\{0\}}:=X_t^{\{0\}}-\theta_\tau^*$ and define
$\delta_{g,2}(t) :=\|g(-U_t)-g(-U_t^{\{0\}})\|_2$.
Since $g$ is $\overline\tau$-Lipschitz, $\delta_{g,2}(t)\le\overline\tau\delta_{X,2}(t)$.
So Assumption~\ref{ass:short-memory} implies 
\begin{equation}\label{eq:score-fdm-transfer}
\begin{aligned}
    \Theta_{g,2}:=\sum_{t=0}^\infty\delta_{g,2}(t)
    \le \overline\tau\Theta_{X,2}<\infty
\end{aligned}
\end{equation}
By \eqref{eq:score-fdm-transfer}, martingale projection orthogonality, and the coupling bound of \citet[Theorem~1(i)--(ii)]{wu2005nonlinear},
\begin{equation*}
    \sum_{h\in\mathbb Z}|\Cov(g(-U_0),g(-U_h))|\le\Theta_{g,2}^2 \le \overline\tau^2 \Theta_{X,2}^2 <\infty .
\end{equation*}
Thus the long-run variance \eqref{eq:score-lrv} converges absolutely and
$\Omega_g<\infty$.

The following lemma bounds the root mean squared distance of the stationary iterate from the target expectile.

\begin{lemma}[Stationary second-moment localization]\label{lem:stationary-localization}
Suppose Assumption~\ref{ass:short-memory} holds, and
$\gamma$ satisfies Assumption~\ref{ass:step-size}. Then
\begin{equation}\label{eq:stationary-localization-order}
\|\theta_t^\circ(\gamma)-\theta_\tau^*\|_2
\le 
\Theta_{X,2}
\sqrt{\Big(1+\frac{\overline\tau^2}{\underline\tau^2}\Big)
\frac{\gamma\overline\tau^2}{2\underline\tau-\gamma\underline\tau^2}}
\end{equation}
\end{lemma}
The proof is given in Appendix~\ref{app:stationary-localization}.
The stationary iterate need not be unbiased for $\theta_\tau^*$. For fixed $\tau$, the localization bound also controls its bias,
\begingroup
\[
|\E[\theta_t^\circ(\gamma)]-\theta_\tau^*|\le\|\theta_t^\circ(\gamma)-\theta_\tau^*\|_2=O(\sqrt\gamma).
\]
\endgroup
When $\tau=1/2$, the recursion is linear and the stationary mean equals $\theta_\tau^*$.
As $\underline\tau\downarrow0$ with $\gamma$ sufficiently small,
the right-hand side of \eqref{eq:stationary-localization-order} is of order
$\Theta_{X,2}\sqrt\gamma\,\underline\tau^{-3/2}$. Thus the localization bound weakens
for increasingly extreme expectile levels. Proposition~\ref{prop:stationary-solution}
transfers this local scale to a trajectory started at a fixed point once the
initialization has been forgotten; together, the two results provide the
second-moment bounds used in the approximation arguments of
Section~\ref{sec:short-memory-limit}.

\subsection{Short-memory last-iterate central limit theorem}\label{sec:short-memory-limit}

We now establish the central limit theory for the last iterate.  The proof controls the nonlinear remainder created by the kink of $g$, approximates the random transition weights, and identifies the limit of the weighted sum of the dependent score.
To control crossings at the kink of $g$, we impose the following conditional anti-concentration assumption.

\begin{assumption}[Conditional anti-concentration]\label{ass:anti-concentration}
There exists a constant $\overline f<\infty$ such that, almost surely,
\begin{equation}\label{eq:conditional-interval-bound}
\Prob(a<U_t\le b\mid\F_{t-1})\le\overline f (b-a),
\qquad a<b .
\end{equation}
\end{assumption}

A uniformly bounded conditional density is sufficient for Assumption~\ref{ass:anti-concentration}, which ensures that $A=\E[g'(-U_0)]\in[\underline\tau,\overline\tau]$ is well defined.

\begin{theorem}[Short-memory last-iterate central limit theorem]\label{thm:short-memory-main}
Fix $\tau\in(0,1)$ and suppose Assumptions~\ref{ass:short-memory} and \ref{ass:anti-concentration} hold.
Let $\theta_0\in\R$ be fixed, let $\gamma\downarrow0$ subject to Assumption~\ref{ass:step-size}.
Let $n=n(\gamma)\to\infty$ satisfy the burn-in condition
\begin{equation}\label{eq:burn-in-main}
n\gamma\ge\nu\log(\gamma^{-1})
\qquad\text{for some }\nu>\frac1{2\underline\tau} .
\end{equation}
Then
\begin{equation}\label{eq:short-memory-main-limit}
\frac{\theta_n(\gamma)-\theta_\tau^*}{\sqrt\gamma}\Rightarrow\N\Big(0,\frac{\Omega_g}{2A}\Big).
\end{equation}
\end{theorem}

Theorem~\ref{thm:short-memory-main} gives a Gaussian limit for the last-iterate error at scale $\sqrt\gamma$ under short-range dependence, with variance $\Omega_g/(2A)$. In the case of i.i.d.\ $(X_k)$, \(\Omega_g\) reduces to the variance of the score. The curvature $A$ determines the strength with which these accumulated score fluctuations are damped by the recursion.

\begin{remark}[Burn-in]\label{rmk:short-memory-burnin}
Condition~\eqref{eq:burn-in-main} ensures that the initialization is negligible at the $\sqrt\gamma$ scale. The threshold uses the uniform contraction rate $1-\gamma\underline\tau$ and is therefore a sufficient worst-case condition. Its logarithmic dependence on $\gamma^{-1}$ is consistent with last-iterate burn-in conditions for constant-step stochastic approximation; see \citet{wei2025gaussian}.

The same limit holds for an $\F_0$-measurable initialization $\theta_0\in L^2$ that does not depend on $\gamma$. The calculations are given in Appendix~\ref{app:burn-in-calculations}.
\end{remark}

\begingroup

We compare the nonlinear recursion with the random-coefficient linearization, initialized at $e_0$ and driven by the same score sequence:
\begin{equation}\label{eq:linearized-recursion}
e_0^{\lin}=e_0,
\quad e_t^{\lin}=(1-\gamma g'(-U_t))e_{t-1}^{\lin}-\gamma g(-U_t)
\end{equation}
Here $R_t:=g(e_{t-1}-U_t)-g(-U_t)-g'(-U_t)e_{t-1}$ is the nonlinear remainder. By piecewise linearity of $g$, it vanishes unless $U_t$ lies between $0$ and $e_{t-1}$, and
\begin{equation}\label{eq:crossing-bound}
|R_t|\le|1-2\tau|\,|e_{t-1}| \1\big(\min(0,e_{t-1})\le U_t\le\max(0,e_{t-1})\big).
\end{equation}
Since $e_{t-1}$ is $\F_{t-1}$-measurable, Assumption~\ref{ass:anti-concentration} yields
\begin{equation}\label{eq:conditional-remainder-bound}
\E[|R_t|\mid\F_{t-1}]\le\overline f|1-2\tau|\,|e_{t-1}|^2.
\end{equation} Proposition~\ref{prop:nonlinear-linearization} makes the nonlinear remainder negligible at scale $\sqrt\gamma$, so it suffices to identify the limit of $e_n^{\lin}/\sqrt\gamma$.

Unrolling \eqref{eq:linearized-recursion} gives
\begin{equation}\label{eq:linearized-solution}
e_n^{\lin}=K_{0,n}e_0-\gamma\sum_{t=1}^nK_{t,n}g(-U_t),
\end{equation}
where $K_{t,n}:=\prod_{j=t+1}^n(1-\gamma g'(-U_j))$. Under dependent observations, this weight need not be independent of the score it multiplies. Lemma~\ref{lem:correlated-transfer} transfers the score functional central limit theorem to the weighted sum without requiring this independence. The burn-in condition makes $K_{0,n}e_0$ negligible at the $\sqrt\gamma$ scale, while the transition weights converge on rescaled time windows to the exponential filter determined by $A$. This yields the variance $\Omega_g/(2A)$ in Theorem~\ref{thm:short-memory-main}. The nonlinear-remainder bound and detailed derivation are given in Appendix~\ref{app:short-memory-derivation}, with the theorem proof in Appendix~\ref{app:short-memory-main}.
\endgroup

\section{Long-Memory Processes}\label{sec:long_memory}

In this section, we specialize the Bernoulli-shift representation \eqref{eq:wold-representation} to a linear process with regularly varying coefficients, whose physical dependence coefficients are also regularly varying. We then record the functional limit of the score process it produces (Proposition~\ref{prop:long-memory-FCLT} in Appendix~\ref{app:long-memory-fclt}) and transfer that limit through the recursion to the actual last iterate (Section~\ref{subsec:long-memory-iterate}).
Section~\ref{subsec:joint-long-memory} gives the joint limit \eqref{eq:joint-long-memory-limit} for finitely many expectile levels $\tau_1,\ldots,\tau_k$ and compares it with batch estimation.

\subsection{Long-memory linear process}\label{subsec:long-memory-model}

\begin{assumption}[Long-memory linear input]\label{ass:long-memory}
The following conditions hold.
\begin{enumerate}[label=(A\arabic*),leftmargin=*]
\item\label{ass:lm-linear} The observations admit the linear representation
\begin{equation}\label{eq:long-memory-linear-process}
X_t=\sum_{j=0}^\infty c_j\varepsilon_{t-j},
\quad c_j\sim j^{-\beta}\ell(j)
\end{equation}
as $j\to \infty$, where $1/2<\beta<1$ and $c_0\neq0$, and $\ell$ is positive and slowly varying.
\item\label{ass:lm-innovation} The innovations $(\varepsilon_t)_{t\in\mathbb Z}$ are i.i.d.\ with 
$\E[\varepsilon_0]=0$, $\E[\varepsilon_0^2] = 1$ and $\E[|\varepsilon_0|^6]<\infty$, and $\varepsilon_0$ has a bounded continuous density $f_\varepsilon$.
\end{enumerate}
\end{assumption}

\begin{remark}
A positive function $\ell$ is slowly varying if $\ell(\lambda x)/\ell(x)\to1$ as $x\to\infty$ for every fixed $\lambda>0$.
The scale $s_\gamma:=\gamma^{\beta-1/2}\ell(1/\gamma)$ normalizes the estimation error $\theta_n(\gamma)-\theta_\tau^*$.
Since $X_t=c_0\varepsilon_t+\sum_{j\ge1}c_j\varepsilon_{t-j}$ and the sum is $\F_{t-1}$-measurable, the conditional law of $X_t$ is a shift of that of $c_0\varepsilon_t$. Hence $\|f_{X_t\mid\F_{t-1}}\|_\infty\le\|f_\varepsilon\|_\infty/|c_0|$, verifying Assumption~\ref{ass:anti-concentration}.
\end{remark}

Under Assumption~\ref{ass:long-memory}, the observation covariances at lags $h\geq0$ satisfy
\begin{equation}\label{eq:long-memory-covariance}
\Cov(X_0,X_h)
=\E[\varepsilon_0^2]\sum_{j=0}^{\infty}c_jc_{j+h}
\asymp h^{1-2\beta}{\ell(h)^2}.
\end{equation}
The convolution argument is given in Appendix~\ref{app:long-memory-localization}.
Since $1-2\beta\in(-1,0)$, these covariances are not summable and the observation process has no finite long-run variance.

Moreover, replacing $\varepsilon_0$ by an independent copy gives, for the coefficients $c_k$ in \eqref{eq:long-memory-linear-process},
\begin{equation}\label{eq:long-memory-fdm}
\delta_{U,2}(k)=|c_k|\,\|\varepsilon_0-\varepsilon_0'\|_2
\asymp k^{-\beta}{\ell(k)},
\end{equation}
which is \emph{not} summable, so Assumption~\ref{ass:short-memory} does not hold.
The stationary solution of Proposition~\ref{prop:stationary-solution} still exists, since that result needs no dependence condition. The following lemma provides a localization bound under long memory.

\begin{lemma}[Stationary localization under long memory]\label{lem:long-memory-stationary-localization}
Suppose Assumption~\ref{ass:long-memory} holds and $\gamma$ satisfies Assumption~\ref{ass:step-size}. Then
\begin{equation}\label{eq:long-memory-stationary-localization}
\E[|\theta_t^\circ(\gamma)-\theta_\tau^*|^2]
\lesssim s_\gamma^2,
\qquad \gamma\downarrow0 .
\end{equation}
\end{lemma}

The stationary solution therefore remains localized around $\theta_\tau^*$, with an $L^2$ bound of order $s_\gamma$. This scale is larger than $\sqrt\gamma$ because $\beta<1$.

\subsection{The long-memory last-iterate limit}\label{subsec:long-memory-iterate}

Let $\mathbb{B}$ be a standard two-sided Brownian motion. Following the normalization of \citet{wu2006unit}, define
\begin{equation}\label{eq:fBM}
    B_H(t) 
    := \frac{1}{1-\beta}\int_{-\infty}^{t} \Big((t-u)^{1-\beta}-(-u)_+^{1-\beta}\Big)\,d\mathbb{B}(u),
\end{equation}
a fractional Brownian motion with Hurst index $H = 3/2-\beta$, with its variance
\begin{align*}
    \sigma_H^2 &:=\E B_H(1)^2\\
    &= \frac{1}{(1-\beta)^{2}}\int_{\mathbb{R}} \Big((1-u)_+^{1-\beta}-(-u)_+^{1-\beta}\Big)^{2}du .
\end{align*}

Under Assumption~\ref{ass:long-memory}, Proposition~\ref{prop:long-memory-FCLT} gives the invariance principle
\[
\frac{\gamma}{s_\gamma}\sum_{k=1}^{\lfloor s/\gamma\rfloor}g(-U_k)\Rightarrow -AB_H(s)
\]
in $D([0,T])$ with the $J_1$ topology as $\gamma\downarrow0$, for every fixed $T>0$. Together with the partial-sum moment bound in Proposition~\ref{prop:long-memory-FCLT}, this gives the limiting score contribution $A\int_0^\infty e^{-As}\,dB_H(s)$ to the linearized recursion via Lemma~\ref{lem:correlated-transfer}. Localization and the conditional density bound make the nonlinear remainder negligible at scale $s_\gamma$, while burn-in removes the initialization term.

\begin{theorem}[Long-memory last-iterate limit]\label{thm:long-memory-last-iterate}
Suppose Assumption~\ref{ass:long-memory} holds, let $H:=3/2-\beta$, and let $\gamma\downarrow0$ and $n=n(\gamma)$ satisfy
\begin{equation}\label{eq:long-memory-burn-in}
\frac{(1-\gamma\underline\tau)^n}{s_\gamma}
\to0,
\end{equation}
for which $n\gamma\ge\nu\log(\gamma^{-1})$ with $\nu>(\beta-1/2)/\underline\tau$ is sufficient.
Then with $\theta_0 \in \R$ fixed,
\begin{equation}
\frac{\theta_n-\theta_\tau^*}{s_\gamma} \Rightarrow \N\left(0,\sigma_H^2 H \Gamma(2 H) A^{2-2 H}\right).
\end{equation}
\end{theorem}

For these linear long-memory inputs, the limit remains Gaussian, but dependence enlarges the localization scale from $\sqrt\gamma$ to $s_\gamma$.
The proof is given in Appendix~\ref{app:long-memory-last-iterate}.

\subsection{Joint limits across expectile levels}\label{subsec:joint-long-memory}

The marginal (single-level) limits of Theorems~\ref{thm:short-memory-main} and~\ref{thm:long-memory-last-iterate} do not determine the uncertainty of differences between expectile levels estimated from the same observations.
Fix finitely many levels $\tau_1,\ldots,\tau_k\in(0,1)$. Write $g_i:=g_{\tau_i}$ and $\theta_i^*:=\theta_{\tau_i}^*$, and let $\theta_t^{(i)}$ follow \eqref{eq:expectile-sgd} at level $\tau_i$ on the same observation sequence with the same learning rate $\gamma$. Put $A_i:=\E[g_i'(\theta_i^*-X_0)]$ and $\underline\tau_i:=\min(\tau_i,1-\tau_i)$.

\begin{theorem}[Joint long-memory law]\label{thm:joint-long-memory}
Suppose Assumption~\ref{ass:long-memory} holds, and let the initial values $\theta_0^{(i)}$ be fixed. Let $\gamma\downarrow0$ and $n=n(\gamma)$ satisfy
\eqs{\max_{1\le i\le k}\frac{(1-\gamma\underline\tau_i)^n}{s_\gamma}\to0.\label{eq:joint-long-memory-burn-in}}
Put $Y_i:=A_i\int_0^\infty e^{-A_i s}\,dB_H(s)$. Then
\begin{equation}\label{eq:joint-long-memory-limit}
\frac1{s_\gamma}\begin{pmatrix}\theta_n^{(1)}-\theta_1^*\\\vdots\\\theta_n^{(k)}-\theta_k^*\end{pmatrix}\Rightarrow\begin{pmatrix}Y_1\\\vdots\\Y_k\end{pmatrix}.
\end{equation}
Every coordinate is a deterministic linear functional of the same $B_H$, so the limit is a jointly Gaussian, centered random vector.
\end{theorem}

The common fractional Brownian motion is filtered through $A_i e^{-A_i s}$ at level $\tau_i$. These integrands determine the dependence between the limiting coordinates.

\begin{corollary}[Covariance and rank]\label{cor:joint-long-memory-covariance}
The limiting random vector in Theorem~\ref{thm:joint-long-memory} has covariance
\eqs{
    \Sigma^H_{ij}
    =\sigma_H^2H\Gamma(2H)\frac{A_iA_j(A_i^{1-2H}+A_j^{1-2H})}{A_i+A_j}.\label{eq:joint-long-memory-covariance}
}
This follows from the deterministic-integrand isometry for fBm \citep[Section~2]{Nualart2006Stochastic}, as calculated in Appendix~\ref{app:joint-limit}.
Its rank is the number of distinct values among $A_1,\ldots,A_k$. In particular, $\Sigma^H$ is positive definite when these population curvatures are distinct, and equal population curvatures give identical limiting coordinates.
\end{corollary}

A pairwise contrast is the difference between two estimation errors, $(\theta_n^{(i)}-\theta_n^{(j)})-(\theta_i^*-\theta_j^*)$. After division by $s_\gamma$, its limit has positive variance exactly when $A_i\ne A_j$. Distinct levels alone need not give distinct population curvatures; for example, complementary levels of a symmetric marginal distribution can have equal population curvatures. The proofs are given in Appendix~\ref{app:joint-limit}.

\paragraph{Comparison with batch expectiles.}
Related long-memory asymptotics for batch expectiles are studied in \citet{ren2026asymptotics}.
For comparison, let $\textcolor{black}{\hat\theta}_{m,i}^{\rm batch}$ denote the batch expectile at level $\tau_i$, defined by $\frac{1}{m}\sum_{t=1}^m g_i(\textcolor{black}{\hat\theta}_{m,i}^{\rm batch}-X_t)=0$,
and write $\bar X_m:=m^{-1}\sum_{t=1}^m X_t$ and $a_m:=m^{H-1}\ell(m)$. 
Under Assumption~\ref{ass:long-memory}, Appendix~\ref{app:joint-batch-comparison} shows that, jointly over the fixed levels,
\eqs{
    \textcolor{black}{\hat\theta}_{m,i}^{\rm batch}-\theta_i^*
    =\bar X_m+o_P(a_m).
    \label{eq:joint-batch-reduction}
}
Thus all batch expectiles share the same first-order long-memory error, and the normalized error vector converges to
$(B_H(1),\ldots,B_H(1))$, with covariance
$\sigma_H^2\mathbf{1}\mathbf{1}^\top$. Consequently, the common leading term cancels in centered cross-level contrasts, which are $o_P(a_m)$. By contrast, for SGD with a common learning rate,
\[
    \frac{
    (\theta_n^{(i)}-\theta_n^{(j)})
    -(\theta_i^*-\theta_j^*)
    }{s_\gamma}
    \Rightarrow
    \N\!\left(
        0,\,
        \Sigma^H_{ii}+\Sigma^H_{jj}-2\Sigma^H_{ij}
    \right),
\]
which is nondegenerate whenever $A_i\ne A_j$. Hence the level-specific exponential filters preserve first-order cross-level fluctuations that cancel for the corresponding batch expectiles.

\section{Simulation Experiments}\label{sec:experiment}

We illustrate the finite-sample behavior of the last iterate under both short- and long-range dependence.  The stationary short-memory models, all with standard Gaussian innovations, are
\begin{align*}
\text{AR(1): }&X_t=0.5X_{t-1}+\varepsilon_t;\\
\text{nonlinear AR: }&X_t=0.9|X_{t-1}|+\varepsilon_t;\\
\text{GARCH$(1,1)$: }&X_t=\sigma_t\varepsilon_t,\\
&\sigma_t^2=0.1+0.1X_{t-1}^2+0.8\sigma_{t-1}^2.
\end{align*}
The first two processes are contractive causal processes, while the third introduces conditional heteroskedasticity.  The nonlinear AR has an asymmetric stationary marginal.  For long memory, we use an ARFIMA$(0,d,0)$ process with $d=0.25$,
\[
    X_t=\sum_{j=0}^{\infty}c_j\varepsilon_{t-j},
    \qquad
    c_j=\frac{\Gamma(j+d)}
               {\Gamma(j+1)\Gamma(d)}.
\]
Since $c_j\asymp j^{d-1}$, this corresponds to the setting of Section~\ref{sec:long_memory} with $\beta=1-d=0.75$.  We use $\tau = 0.9$ for Figure~\ref{fig:kolmogorov}(a)--(d) and $\tau=0.5,0.9$ for MSE with varying constant learning rates $\gamma$.

Let $Z_{t,\gamma}$ denote the SGD error standardized according to the Gaussian approximations in Theorems~\ref{thm:short-memory-main} and \ref{thm:long-memory-last-iterate}. Simulation settings and calibration details are given in Appendices~\ref{app:numerical-setup} and \ref{app:numerical-calibration}.

Figure~\ref{fig:kolmogorov}(a)--(d) examines the Gaussian
approximation along the SGD trajectory.  At each checkpoint we estimate the Kolmogorov distance from $30{,}000$ independent trajectories.
Across all four processes, the distance fluctuates after the transient period around a plateau that is generally lower for smaller learning rates. Smaller learning rates require longer transient periods, consistent with the joint $n$--$\gamma$ asymptotics of Theorems~\ref{thm:short-memory-main} and \ref{thm:long-memory-last-iterate}.
At $\gamma=10^{-4}$, the last-iterate distances for AR(1), GARCH, and ARFIMA all fall within the shaded 95\% Monte Carlo reference band.

\begin{figure}[!ht]
    \centering
    \includegraphics[width=\textwidth]{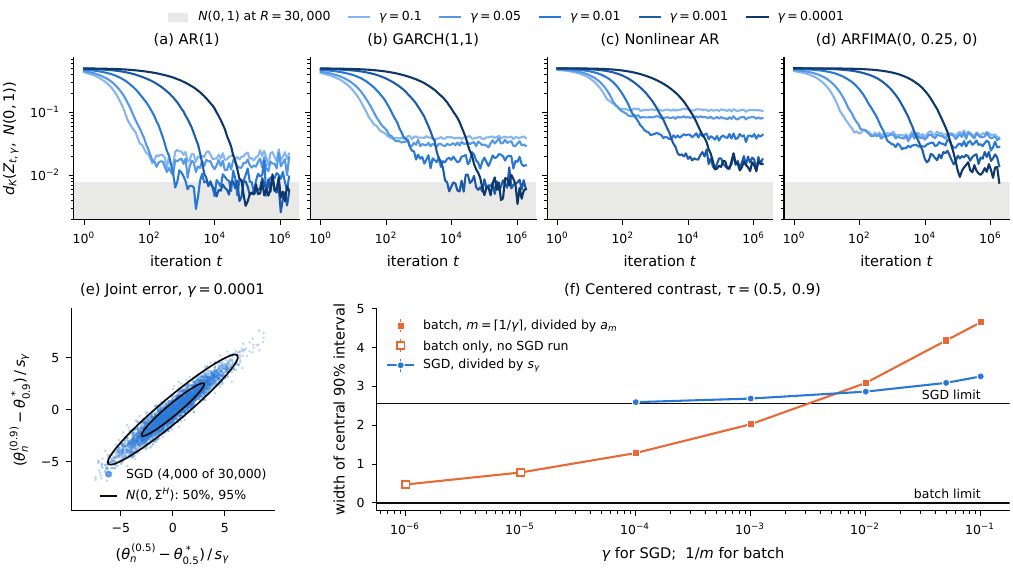}
    \caption{
    (a)--(d) Kolmogorov distance to the corresponding Gaussian limit
    along the expectile SGD trajectory for AR(1), GARCH$(1,1)$, nonlinear AR,
    and ARFIMA$(0,0.25,0)$ processes at $\tau=0.9$.
    Each curve corresponds to a different learning rate $\gamma$.
    The shaded region shows the 95\% Monte Carlo reference band for the estimated Kolmogorov distance under the $\N(0,1)$ reference, excluding calibration uncertainty.
    (e) Joint normalized SGD errors at $\tau=0.5,0.9$ and $\gamma=10^{-4}$ under ARFIMA$(0,0.25,0)$ with 50\% and 95\% contours of the joint Gaussian limit.
    (f) Central 90\% interval widths of the normalized centered contrast between $\tau=0.5$ and $0.9$ under ARFIMA$(0,0.25,0)$. Batch uses the last $m=\lceil1/\gamma\rceil$ observations, with open squares for $m=10^5,10^6$. Error bars are pointwise 95\% bootstrap intervals.
    }
    \label{fig:kolmogorov}
    \label{fig:joint-contrast}
\end{figure}

Under short memory, the empirical MSE is approximately linear in $\gamma$, whereas for ARFIMA$(0,0.25,0)$ it scales approximately as $\gamma^{1-2d}=\gamma^{1/2}$. These trends agree with the squared fluctuation scales in Theorems~\ref{thm:short-memory-main} and \ref{thm:long-memory-last-iterate}. At $\gamma=10^{-4}$ and $\tau=0.9$, the nonlinear AR MSE is within $1.3\%$ of the predicted Gaussian variance, but its Kolmogorov distance remains above the Monte Carlo reference. The reference band excludes calibration uncertainty, which may contribute to the nonlinear-AR discrepancy. Table~\ref{tab:numerical-mse} in Appendix~\ref{app:numerical-marginal} reports the MSE ratios and fitted log--log slopes.

We next examine the joint distribution of the normalized SGD errors at $\tau=0.5,0.9$ on shared ARFIMA paths. Figure~\ref{fig:joint-contrast}(e) compares these errors at $\gamma=10^{-4}$ with the Gaussian contours from Theorem~\ref{thm:joint-long-memory}. The elongated cloud shows strong positive dependence between the errors on the common observation paths. The population curvatures at these two levels differ, so their joint limit has a nonzero contrast variance despite the shared input.

Figure~\ref{fig:joint-contrast}(f) illustrates the different first-order behavior of SGD and batch contrasts. After normalization by their respective long-memory scales, the SGD interval width approaches the positive limit $2.561$ from Corollary~\ref{cor:joint-long-memory-covariance}, while the batch width decreases toward zero as the window grows. Normalizations and interval construction are detailed in Appendix~\ref{app:numerical-joint}.

\section{Conclusion}\label{sec:conclusion}
In this paper, we established distributional limits for the last iterate of constant-step expectile SGD under stationary dependent observations. 
Under short memory, the last iterate has a Gaussian limit at scale $\sqrt{\gamma}$, with variance determined by the long-run variance of the score process. For long-memory linear processes, we obtain the larger scale $\gamma^{\beta-1/2}\ell(1/\gamma)$ and a Gaussian limit expressed as an exponentially weighted fractional Brownian motion integral. Under long memory, the joint SGD limit has full-rank covariance when the population curvatures are distinct, whereas the batch limit has rank one, so centered cross-level contrasts can retain first-order fluctuations for SGD that cancel for batch expectiles.

Future work includes extending the last-iterate theory to multivariate expectile estimation and studying extreme expectiles as \(\tau\uparrow1\) and \(\gamma\downarrow0\) jointly. In this regime, $\underline\tau\to0$, weakening contraction and enlarging the localization bound \eqref{eq:stationary-localization-order} through the factor $\underline\tau^{-3/2}$. The present fixed-level bounds therefore do not directly yield this joint limit. Another direction is to characterize last-iterate limits under nonlinear long-memory processes. This extension requires new control of the score partial sums, whose normalization and limiting process may depend on the nonlinear transformation. For functions of Gaussian long-memory processes, Hermite-rank methods provide one possible approach.

\bibliographystyle{plainnat}
\bibliography{references.bib}

\clearpage
\appendix
\section*{Appendix}

Throughout the appendix, $P_tZ:=\E[Z\mid\F_t]-\E[Z\mid\F_{t-1}]$ denotes the martingale projection operator associated with the innovation filtration introduced after \eqref{eq:wold-representation}; it decomposes a causal $L^2$ random variable into its innovation contributions.

\section{Proof for Stationary Solution}
\label{app:proof_statioanry}
\subsection{Proof of Proposition~\ref{prop:stationary-solution} }
\label{app:proof_propgmc}

\begin{proof}

Recall from Section~\ref{sec:stability} that, for every $x$, the map $F_x$
is Lipschitz in $\theta$ with contraction factor
$1-\gamma\underline\tau$ under Assumption~\ref{ass:step-size}.

\textbf{Step 1.}
We establish the existence of the backward limit. 
For $m\ge 1$, the contraction property gives
\[
\begin{aligned}
    |B_{m+1}(\theta)-B_m(\theta)| &\le (1-\gamma\underline\tau)^m |F_{X_{-m}}(\theta)-\theta|\\
    &\le \gamma\overline\tau (1-\gamma\underline\tau)^m (|\theta|+|X_{-m}|).
\end{aligned}
\]
The second inequality follows from
$|F_x(\theta)-\theta| = \gamma|g(\theta-x)| \le
\gamma\overline\tau(|\theta|+|x|)$, since $g$ is
$\overline\tau$-Lipschitz.
Taking $L^p$ norms and using stationarity,
\[
    \|B_{m+1}(\theta)-B_m(\theta)\|_p \le \gamma\overline\tau (1-\gamma\underline\tau)^m (|\theta|+\|X_0\|_p).
\]
Summing along a geometric series,
\[
    \sum_{m=M}^\infty \|B_{m+1}(\theta)-B_m(\theta)\|_p \le \gamma\overline\tau (|\theta|+\|X_0\|_p) \frac{(1-\gamma\underline\tau)^M}{\gamma\underline\tau} <\infty.
\]
Therefore, for $N>M$,
\[
    \|B_N(\theta)-B_M(\theta)\|_p \le \sum_{m=M}^{N-1} \|B_{m+1}(\theta)-B_m(\theta)\|_p \to 0
\]
as $M,N\to\infty$. Thus $\{B_m(\theta)\}_{m\ge1}$ is Cauchy in $L^p$.

Now observe that
\[
\left\| \sum_{m=M}^\infty |B_{m+1}(\theta) - B_m(\theta)| \right\|_1 \leq \sum_{m=M}^\infty \|B_{m+1}(\theta) - B_m(\theta)\|_p < \infty
\]

and therefore
\[
    \sum_{m=1}^\infty |B_{m+1}(\theta)-B_m(\theta)| <\infty
    \qquad\text{a.s.}
\]
Then by the same reasoning, the backward iterates are almost surely Cauchy. 
Since $\R$ is complete and $L^p$ is complete, there exists a random variable
$\theta_0^\circ\in L^p$ such that
\[
    B_m(\theta) \to \theta_0^\circ \quad\text{a.s. and in }L^p.
\]
For any $\theta \in \R$,
$|B_m(\theta)-B_m(0)| \le (1-\gamma\underline\tau)^m|\theta|$.
Letting $m\to\infty$ shows that the limiting random variable does not depend on the starting value.

Since 
\[
    \theta_0^\circ
    =B_1(\theta)+\sum_{m=1}^\infty
      \bigl(B_{m+1}(\theta)-B_m(\theta)\bigr),
\]

taking $L^p$ norms and applying the triangle inequality yields

\[
    \|\theta_0^{\circ}\|_p \le \|B_1(\theta)\|_p+\sum_{m=1}^{\infty}\|B_{m+1}(\theta)-B_m(\theta)\|_p<\infty.
\]

\textbf{Step 2.}
Apply the same backward construction at each time $t\in\mathbb Z$ to define $\theta_t^\circ$ as a measurable function of $(X_t,X_{t-1},\ldots)$.
Stationarity of $(X_t)_{t\in\mathbb Z}$ implies that
$(\theta_t^\circ)_{t\in\mathbb Z}$ is stationary. Continuity of
$F_{X_t}$ gives
\[
    \theta_t^\circ
    =F_{X_t}\bigl(\theta_{t-1}^\circ\bigr).
\]

To prove uniqueness, let $(\widetilde\theta_t^\circ)_{t\in\mathbb Z}$ be another
stationary causal $L^p$ solution driven by the same process
$(X_t)_{t\in\mathbb Z}$.
Then, for every $m\ge1$,
\[
\begin{aligned}
    |\theta_t^\circ-\widetilde\theta_t^\circ|
    &\le (1-\gamma\underline\tau)^m
    |\theta_{t-m}^\circ-\widetilde\theta_{t-m}^\circ|.
\end{aligned}
\]
Taking $L^p$ norms gives
\[
\begin{aligned}
    \|\theta_t^\circ-\widetilde\theta_t^\circ\|_p
    &\le (1-\gamma\underline\tau)^m
    \|\theta_{t-m}^\circ-\widetilde\theta_{t-m}^\circ\|_p.
\end{aligned}
\]
By stationarity and the finite $p$th moments,
$\|\theta_{t-m}^\circ-\widetilde\theta_{t-m}^\circ\|_p
\le \|\theta_0^\circ\|_p+\|\widetilde\theta_0^\circ\|_p<\infty$.
Then letting $m\to\infty$ yields
\[
    \|\theta_t^\circ-\widetilde\theta_t^\circ\|_p=0,
\]
and hence $\theta_t^\circ=\widetilde\theta_t^\circ$ almost surely.

\textbf{Step 3.}
Finally, the forward recursion and the stationary recursion use the same
"future observations" $(X_t)_{t \geq 1}$. Therefore,

\[
    |\theta_t-\theta_t^\circ|
    \le (1-\gamma\underline\tau)^t |\theta_0-\theta_0^\circ|.
\]

Taking $L^p$ norms proves the stated geometric forgetting. Since
$(1-\gamma\underline\tau)^t\to0$ and
$\theta_t^\circ\sim\pi_\gamma$ for every $t$, it follows that
\[
    \theta_t\Rightarrow\pi_\gamma.
\]
Moreover, $\pi_\gamma$ has finite $p$th moment because
$\theta_0^\circ\in L^p$, as shown above.

\end{proof}

\subsection{Proof of Lemma~\ref{lem:stationary-localization}}
\label{app:stationary-localization}

Throughout this subsection $e_t^\circ:=\theta_t^\circ(\gamma)-\theta_\tau^*$ denotes the stationary error, so that by \eqref{eq:centered-recursion} it satisfies $e_t^\circ=e_{t-1}^\circ-\gamma g(e_{t-1}^\circ-U_t)$ with $\E[g(-U_t)]=0$.  We also write $\Theta_{U,2}:=\sum_{k\ge0}\delta_{U,2}(k)$, which equals $\Theta_{X,2}$ because $U_t=X_t-\theta_\tau^*$ gives $\delta_{U,2}(k)=\delta_{X,2}(k)$.

\begin{proof}
We bound the variance and squared mean of the stationary error.
\[
    \E[|e_0^\circ|^2] =\Var(e_0^\circ) + |\E[e_0^\circ]|^2.
\]

\textbf{Term 1: centered second moment.}
Let $(e_t^{\circ,\{0\}})$ denote the stationary solution driven by the coupled input $(U_t^{\{0\}})$ obtained by replacing $\varepsilon_0$ with an independent copy. 
Since the two inputs coincide before time $0$, $e_{-1}^{\circ,\{0\}}=e_{-1}^\circ$.

For every $k\ge0$, write $\delta_{e, 2} (k) := \|e_k^{\circ}-e_k^{\circ,\{0\}}\|_2$. 
Since the map $F_U$ is $(1 - \gamma \underline \tau)$-Lipschitz and the score function $g$ is $\overline\tau$-Lipschitz,
\begin{equation}
\begin{aligned}[b]
\delta_{e, 2} (k)
&=\|F_{U_k}(e_{k-1}^{\circ})-F_{U_k^{\{0\}}}(e_{k-1}^{\circ,\{0\}})\|_2  \\
&\leq \|F_{U_k}(e_{k-1}^{\circ})-F_{U_k}(e_{k-1}^{\circ,\{0\}})\|_2
    +\|F_{U_k}(e_{k-1}^{\circ,\{0\}})-F_{U_k^{\{0\}}}(e_{k-1}^{\circ,\{0\}})\|_2  \\
&\leq (1-\gamma\underline\tau)\|e_{k-1}^{\circ}-e_{k-1}^{\circ,\{0\}}\|_2
    +\gamma\overline\tau\delta_{U,2}(k)\\
&\le \cdots \le \gamma\overline\tau \sum_{\ell=0}^{k} (1 - \gamma \underline \tau)^{k-\ell}\delta_{U,2}(\ell),
\end{aligned}
\label{eq:e-fdm-propagation}
\end{equation}
Since $e_0^\circ$ is a causal $L^2$ random variable by Proposition~\ref{prop:gmc}, the martingale projection decomposition yields
\[
\Var(e_0^\circ) = \sum_{j=0}^\infty \|P_{-j}e_0^\circ\|_2^2
= \sum_{k=0}^\infty \|P_0e_k^\circ\|_2^2 \le \sum_{k=0}^\infty\delta_{e,2}(k)^2,
\]
where the last inequality holds by \cite{wu2005nonlinear}.

Applying Young's convolution inequality $\ell^1*\ell^2\to\ell^2$ to \eqref{eq:e-fdm-propagation},
\[
\Bigg( \sum_{k=0}^\infty\delta_{e,2}(k)^2 \Bigg)^{1/2} 
\le \Bigg( \sum_{\ell=0}^\infty\delta_{U,2}(\ell) \Bigg) 
\Bigg(\sum_{j=0}^\infty \gamma^2\overline\tau^2 (1-\gamma \underline\tau)^{2j} \Bigg)^{1/2}.
\]
Hence
\begin{equation}
\Var(e_0^\circ) 
\le \Theta_{U,2}^2 \frac{\gamma^2\overline\tau^2}{1-(1 - \gamma \underline\tau)^2} 
= \frac{\gamma\overline\tau^2} {2\underline\tau-\gamma\underline\tau^2} \Theta_{U,2}^2.
\label{eq:stationary-e-var}
\end{equation}

\textbf{Term 2: stationary mean.}
Define the random secant slope
\[
    b_t:=\begin{cases}
    \displaystyle\frac{g(e_{t-1}^\circ-U_t)-g(-U_t)}{e_{t-1}^\circ}, & e_{t-1}^\circ\ne0,\\[1.2ex]
    \underline\tau, & e_{t-1}^\circ=0.
    \end{cases}
\]
Then $\underline\tau\le b_t\le\overline\tau$ almost surely.  By stationarity and $\E[g(-U_t)]=0$
\begin{align*}
    0&=\E[e_t^\circ-e_{t-1}^\circ]=-\gamma \E[g(e_{t-1}^\circ-U_t)]\\
    &= -\gamma \E[e_{t-1}^\circ b_t + g(-U_t)] = -\gamma \E[e_{t-1}^\circ b_t].
\end{align*}

Then from $b_te_{t-1}^\circ = b_t(e_{t-1}^\circ - \E e_{t-1}^\circ) + \E [ e_{t-1}^\circ] b_t$,
\begin{align*}
    |\E[e_0^\circ]| \E[b_t]
    &=|\E[b_t(e_{t-1}^\circ-\E[e_0^\circ])]| \\
    &\le \E [| b_t | |e_{t-1}^\circ - \E[e_0^\circ]|]\\
    &\le \overline\tau \E[|e_{t-1}^\circ-\E[e_0^\circ]|],
\end{align*}

Therefore, by $\E[b_t]\ge\underline\tau$ and $|b_t|\le\overline\tau$,
\begin{equation}
|\E e_0^\circ| 
\le \frac{\overline\tau}{\underline\tau} \E[|e_{t-1}^\circ-\E e_0^\circ|]
\le \frac{\overline\tau}{\underline\tau} \sqrt{\Var(e_0^\circ)}.
\label{eq:stationary-e-mean}
\end{equation}

Combining \eqref{eq:stationary-e-var} and \eqref{eq:stationary-e-mean},
\begin{align*}
\E[|e_0^\circ|^2] &= \Var(e_0^\circ)+|\E e_0^\circ|^2  \\
&\le \Big( 1+\frac{\overline\tau^2}{\underline\tau^2} \Big) \frac{\gamma\overline\tau^2} {2\underline\tau-\gamma\underline\tau^2} \Theta_{U,2}^2 = O(\gamma).
\end{align*}
\end{proof}
\paragraph{The population score equation.}\label{app:population-score}
Differentiation under the expectation is justified by dominated convergence: $f_\tau'=g$ is Lipschitz, and on every bounded neighborhood of $\theta_\tau^*$, $|g(\theta-X_0)|\leq\overline\tau(|\theta|+|X_0|)$, an integrable bound since $X_0\in L^2$.

\section{Proof for Short-Memory Processes}\label{app:proof_short_memory}

\begingroup

\subsection{Short-memory derivation}\label{app:short-memory-derivation}

Using a first-order expansion of $g$ around $-U_t$, we write the error recursion \eqref{eq:centered-recursion} as
\begin{equation}\label{eq:original-error-expanded}
e_t=(1-\gamma g'(-U_t))e_{t-1}-\gamma g(-U_t)-\gamma R_t,
\end{equation}
where
$R_t:=g(e_{t-1}-U_t)-g(-U_t)-g'(-U_t)e_{t-1}$ is the nonlinear remainder.

\paragraph{Proof outline.}
We compare \eqref{eq:original-error-expanded} with the random-coefficient linearization \eqref{eq:linearized-recursion}, initialized at $e_0$ and driven by the same score sequence.
Equation~\eqref{eq:linearized-recursion} is \eqref{eq:original-error-expanded} without the nonlinear remainder $R_t$. The central decomposition is
\begin{equation}\label{eq:linearized-decomposition}
\frac{e_n}{\sqrt\gamma}=\frac{e_n^{\lin}}{\sqrt\gamma}+\frac{e_n-e_n^{\lin}}{\sqrt\gamma}.
\end{equation}
Proposition~\ref{prop:nonlinear-linearization} proves that the second term converges to zero in $L^1$. Lemma~\ref{lem:correlated-transfer} transfers the score functional central limit theorem (FCLT) through the transition products $K_{i,n}:=\prod_{j=i+1}^n(1-\gamma g'(-U_j))$ to identify the limit of the first term. Together with the burn-in condition \eqref{eq:burn-in-main}, these steps prove Theorem~\ref{thm:short-memory-main}.

\subsubsection{Linearization error}\label{subsec:nonlinear-linearization}

A uniformly bounded conditional density is sufficient for Assumption~\ref{ass:anti-concentration}. For every $h>0$, $\Prob(U_t=0\mid\F_{t-1})\le\overline f h$. Letting $h\downarrow0$ gives $\Prob(X_0=\theta_\tau^*)=0$, so $A=\E[g'(-U_0)]\in[\underline\tau,\overline\tau]$ is well defined.

Because $g$ is absolutely continuous, the fundamental theorem of calculus
gives
\begingroup
\begin{equation}\label{eq:remainder-integral}
R_t=\int_0^{e_{t-1}}\bigl(g'(s-U_t)-g'(-U_t)\bigr)\,ds.
\end{equation}
\endgroup
By piecewise linearity of $g$, its derivative in \eqref{eq:score-derivative} is constant on each side of zero and jumps there by magnitude $|1-2\tau|$. The integrand in \eqref{eq:remainder-integral} therefore vanishes unless $U_t$ lies between $0$ and $e_{t-1}$. This gives~\eqref{eq:crossing-bound}.
Assumption~\ref{ass:anti-concentration} bounds the conditional probability that $U_t$ lies in $[\min(0,e_{t-1}),\max(0,e_{t-1})]$.

In particular, $\E[|R_t|\mid\F_{t-1}]\le\overline f|1-2\tau|\,|e_{t-1}|^2$.

\begin{proposition}[Linearization error]\label{prop:nonlinear-linearization}
Let $\theta_0\in\R$ be fixed and let $\gamma\downarrow0$ subject to Assumption~\ref{ass:step-size}.
Suppose Assumptions~\ref{ass:short-memory} and \ref{ass:anti-concentration} hold, and let $n=n(\gamma)\to\infty$ satisfy the burn-in condition \eqref{eq:burn-in-main}.
Then
\begin{equation}\label{eq:nonlinear-linearization-result}
\frac{e_n-e_n^{\lin}}{\sqrt\gamma}\to0
\qquad\text{in }L^1 .
\end{equation}
\end{proposition}

The proof is given in Appendix~\ref{app:nonlinear-linearization}.

\subsubsection{The last-iterate limit}\label{subsec:random-coefficient-replacement}\label{subsec:short-memory-assembly}

Unrolling \eqref{eq:linearized-recursion} assigns the score $g(-U_t)$ the transition weight $K_{t,n}=\prod_{j=t+1}^n(1-\gamma g'(-U_j))$. Under dependent observations, this weight need not be independent of the score it multiplies. Lemma~\ref{lem:correlated-transfer} in Appendix~\ref{app:transfer} transfers a functional limit for the score partial sums to the corresponding weighted sum without requiring this independence, and applies to both memory regimes.

With $a_t=g'(-U_t)$ and $\xi_t=g(-U_t)$ in Lemma~\ref{lem:correlated-transfer}, unwinding \eqref{eq:linearized-recursion} gives \eqref{eq:linearized-solution}.
The score FCLT \eqref{eq:short-memory-score-fclt} verifies \eqref{eq:transfer-fclt-assumption} with $v(m)=\sqrt m$ and $b_\gamma=\sqrt\gamma$. Lemma~\ref{lem:correlated-transfer} transfers it to the weighted sum in \eqref{eq:linearized-solution}. Its proof in Appendix~\ref{app:transfer} combines summation by parts with convergence of the random weights and control of the remote past. Together with Proposition~\ref{prop:nonlinear-linearization} and \eqref{eq:burn-in-main}, this proves Theorem~\ref{thm:short-memory-main}.

\endgroup

\subsection{Proof of Proposition~\ref{prop:nonlinear-linearization}}
\label{app:nonlinear-linearization}

Assumption~\ref{ass:anti-concentration} makes the conditional distribution function of $U_t$ Lipschitz and hence continuous. The no-atom argument in Section~\ref{subsec:nonlinear-linearization} gives $\Prob(X_0=\theta_\tau^*)=0$. The difference quotients of $g$ are bounded by $\overline\tau$, so dominated convergence gives
\[
\left.\frac{d}{d\theta}\E[g(\theta-X_0)]\right|_{\theta=\theta_\tau^*}=\E[g'(-U_0)]=A,
\]
where $A=\tau\Prob(X_0>\theta_\tau^*)+(1-\tau)\Prob(X_0<\theta_\tau^*)\in[\underline\tau,\overline\tau]$.

\begin{proof}

We can express the difference between the linearized and nonlinear errors by
iterating \eqref{eq:original-error-expanded} and
\eqref{eq:linearized-recursion}:
\begin{equation}\label{eq:linearization-variation-constants}
e_n-e_n^{\lin}=-\gamma\sum_{k=1}^n\Bigg(\prod_{j=k+1}^n(1-\gamma g'(-U_j))\Bigg)R_k .
\end{equation}

We derive the bound on $R_k$ explicitly. Since $g$ is absolutely continuous,
\[
\begin{aligned}
g(e_{k-1}-U_k)-g(-U_k)&=\int_0^{e_{k-1}}g'(s-U_k)\,ds,\\
g'(-U_k)e_{k-1}&=\int_0^{e_{k-1}}g'(-U_k)\,ds.
\end{aligned}
\]
Subtracting gives
\[
R_k=\int_0^{e_{k-1}}\bigl(g'(s-U_k)-g'(-U_k)\bigr)\,ds.
\]
If $U_k$ lies outside the interval between $0$ and $e_{k-1}$, the derivative difference vanishes almost everywhere on the integration interval. Otherwise its magnitude is at most $|1-2\tau|$, the jump of $g'$ at zero. Therefore,
\[
|R_k|\le|1-2\tau|\,|e_{k-1}|\,\1\bigl(\min(0,e_{k-1})\le U_k\le\max(0,e_{k-1})\bigr),
\]
which is \eqref{eq:crossing-bound}. Since $e_{k-1}$ is $\F_{k-1}$-measurable, the interval endpoints are fixed conditionally on $\F_{k-1}$. Applying \eqref{eq:conditional-interval-bound}, with endpoints included by continuity of the conditional distribution, yields
\[
\begin{aligned}
\E[|R_k|\mid\F_{k-1}]&\le|1-2\tau|\,|e_{k-1}|\,\Prob\bigl(\min(0,e_{k-1})\le U_k\le\max(0,e_{k-1})\mid\F_{k-1}\bigr)\\
&\le\overline f|1-2\tau|\,|e_{k-1}|^2.
\end{aligned}
\]
Taking expectations gives
\begin{equation*}
    \E[|R_k|] \leq |1 - 2 \tau| \overline f \E[|e_{k-1}|^2].
\end{equation*}

Because this proof considers the family indexed by $\gamma\downarrow0$, we
write the coupled forward and stationary iterates as $\theta_t(\gamma)$ and
$\theta_t^\circ(\gamma)$, respectively. Coupling them as in
Proposition~\ref{prop:stationary-solution} gives
\begin{align*}
    \E[|e_t|^2]
    &=\E\big[|\theta_t(\gamma)-\theta_t^\circ(\gamma)
              +\theta_t^\circ(\gamma)-\theta_\tau^*|^2\big] \\
    &\leq2\E\big[|\theta_t(\gamma)-\theta_t^\circ(\gamma)|^2\big]
      +2\E\big[|\theta_t^\circ(\gamma)-\theta_\tau^*|^2\big].
\end{align*}

Applying Proposition~\ref{prop:stationary-solution} and Lemma~\ref{lem:stationary-localization} to deal with the first and second terms respectively yields
\begin{equation}\label{eq:forward-second-moment}
\begin{split}
\E[|e_k|^2]
&\leq2(1-\gamma\underline\tau)^{2k}
       \E[|\theta_0(\gamma)-\theta_0^\circ(\gamma)|^2]
       +2\E[|\theta_k^\circ(\gamma)-\theta_\tau^*|^2]\\
&\leq C(1-\gamma\underline\tau)^{2k}+C\gamma,
\end{split}
\end{equation}

where $C$ is independent of $k$ and all sufficiently small $\gamma$.
Consequently,
\begin{equation}\label{eq:forward-remainder-bound}
    \E[|R_k|]
    \leq C(1-\gamma\underline\tau)^{2(k-1)}+C\gamma.
\end{equation}
Since $g'(-U_j)\ge\underline\tau$ and $\gamma\overline\tau<1$ by Assumption~\ref{ass:step-size},
\[
0<\prod_{j=k+1}^n(1-\gamma g'(-U_j))\le(1-\gamma\underline\tau)^{n-k} ,
\]
so that, using \eqref{eq:forward-remainder-bound},
\begin{align*}
\E[|e_n-e_n^{\lin}|]&\le C\gamma\sum_{k=1}^n(1-\gamma\underline\tau)^{n-k}((1-\gamma\underline\tau)^{2(k-1)}+\gamma)\\
&\le C(1-\gamma\underline\tau)^{n-1}+C\gamma .
\end{align*}
Condition \eqref{eq:burn-in-main} gives $(1-\gamma\underline\tau)^n=o(\sqrt\gamma)$, and $\gamma=o(\sqrt\gamma)$, which proves \eqref{eq:nonlinear-linearization-result}.
\end{proof}

\subsection{Transfer of score partial-sum limits}
\label{app:transfer}

Here $D([0,T])$ is the space of right-continuous paths with left limits. Convergence in the Skorokhod $J_1$ topology means uniform convergence after strictly increasing continuous time changes from $[0,T]$ onto itself that converge uniformly to the identity.

\begin{lemma}\label{lem:correlated-transfer}
Let $(a_t,\xi_t)_{t\in\mathbb Z}$ be jointly stationary with $\E\xi_0=0$. Suppose $(a_t)$ is ergodic, $0<\underline a\le a_t\le\overline a<\infty$, and $A:=\E a_0$. Let $v(m)$ be a positive regularly varying function with index $\kappa\in(0,1)$, meaning $v(\lambda m)/v(m)\to\lambda^\kappa$ as $m\to\infty$ for each $\lambda>0$. Assume $\|\sum_{t=1}^m\xi_t\|_2\le Cv(m)$ for all $m\ge1$ and, for every $T>0$,
\eqs{\frac1{v(N)}\sum_{k=1}^{\lfloor Ns\rfloor}\xi_k\Rightarrow Z(s)\quad\text{in }D([0,T]),\ J_1,\label{eq:transfer-fclt-assumption}}
where $Z$ is a possibly zero multiple of fractional Brownian motion with Hurst index $\kappa$. Put $b_\gamma:=\gamma v(1/\gamma)$ and define the transition products $K_{i,n}:=\prod_{j=i+1}^n(1-\gamma a_j)$, with $K_{n,n}=1$. For every deterministic $n=n(\gamma)$ satisfying $n\gamma\to\infty$, as $\gamma\downarrow0$,
\eqs{\frac\gamma{b_\gamma}\sum_{t=1}^nK_{t,n}\xi_t\Rightarrow\int_0^\infty e^{-As}\,dZ(s).\label{eq:transfer-stationary-limit}}
\end{lemma}

\begin{proof}
For sufficiently small $\gamma$, $0<\gamma<1/\overline a$. We use the assumed block bound
\eqs{\Bigg\|\sum_{t=1}^m\xi_t\Bigg\|_2\le Cv(m),\qquad m\ge1,\label{eq:transfer-block-assumption}}
which in particular gives $\xi_0\in L^2$ by stationarity. Fix a deterministic $n=n(\gamma)$ with $n\gamma\to\infty$. Put $q_\gamma:=1-\gamma\underline a$ and define the normalized sum
\eqs{
    \mathcal I_\gamma:=\frac\gamma{b_\gamma}\sum_{t=1}^nK_{t,n}\xi_t.\label{eq:transfer-stationary-filters}
}
For fixed $T>0$, let $m=\lfloor T/\gamma\rfloor$. The proof follows the chain
\[
\begin{aligned}
\mathcal I_\gamma
&=
\underbrace{\frac{\gamma}{b_\gamma}
\sum_{t=n-m+1}^n K_{t,n}\xi_t}_{\text{recent window}}
+
\underbrace{\frac{\gamma}{b_\gamma}
\sum_{t=1}^{n-m}K_{t,n}\xi_t}_{\text{remote past}} \\
&\approx
\frac{\gamma}{b_\gamma}\sum_{k=1}^m K_{k,m}\xi_k
\\
&=
W_\gamma(T)-\int_0^T W_\gamma(s)\,\mu_{\gamma,T}(ds)
\\
&\Rightarrow
\int_0^T e^{-A(T-s)}\,dZ(s)
\overset d=
\int_0^T e^{-As}\,dZ(s),
\end{aligned}
\]
while the remote-past term is $O_{L^2}(e^{-\underline a T})$.
Steps 1--2 establish the kernel approximation, Step 3 the fixed-window
transfer, Step 4 the tail bound, and Step 5 lets $T\to\infty$.

The improper integral in the conclusion is defined by the $L^2$ limit of its finite-interval integration-by-parts expression in Step 5.

Since $v$ has index $\kappa\in(0,1)$, $v(x)\to\infty$ and $v(x)/x\to0$, so $b_\gamma\to0$. All constants below are independent of sufficiently small $\gamma$ and of deterministic time shifts. Constants in the tail bounds are also independent of the truncation time $T$.

\paragraph{Step 1: Uniform control of the centered slope sums.}
Put $\eta_t:=a_t-A$, $S_k^\eta:=\sum_{j=1}^k\eta_j$, and $S_0^\eta:=0$. Ergodicity gives $S_k^\eta/k\to0$ almost surely. To obtain the maximum needed for the products, fix $\epsilon>0$. Almost surely, there is a finite $k_\epsilon$ such that $|S_k^\eta|\le\epsilon k$ for every $k\ge k_\epsilon$. Consequently,
\eqs*{
    \frac1m\max_{0\le k\le m}|S_k^\eta|
    &\le\frac1m\max_{0\le k<k_\epsilon}|S_k^\eta|+\frac1m\max_{k_\epsilon\le k\le m}\epsilon k\\
    &\le\frac1m\max_{0\le k<k_\epsilon}|S_k^\eta|+\epsilon.
}
Letting $m\to\infty$ and then $\epsilon\downarrow0$ proves
\eqs{\frac1m\max_{0\le k\le m}|S_k^\eta|\asto0.\label{eq:transfer-ergodic-maximum}}
For fixed $T>0$, set $m:=\lfloor T/\gamma\rfloor$. Since $\gamma m\to T$,
\eqs{\gamma\max_{0\le k\le m}|S_k^\eta|=(\gamma m)\frac1m\max_{0\le k\le m}|S_k^\eta|\asto0.\label{eq:transfer-scaled-maximum}}
Stationarity gives the same convergence in probability for any deterministically shifted block of length $m$:
\[
\max_{k\le m}\left|\sum_{j=1}^k(a_{r+j}-A)\right|
\overset d= \max_{k\le m}|S_k^\eta|
\]

\paragraph{Step 2: Uniform approximation of the transition products.}
For $1\le k\le m$, use the products $K_{k,m}$ from the statement and set $h_T(s):=e^{-A(T-s)}$, $0\le s\le T$. For $0\le x\le1/2$,
\eqs*{\log(1-x)+x=-\int_0^x\frac{u}{1-u}\,du,\qquad |\log(1-x)+x|\le\int_0^x2u\,du=x^2.}
Thus, for $\gamma\overline a\le1/2$,
\eqs{
    \log K_{k,m}
    &=\sum_{j=k+1}^m\log(1-\gamma a_j)
    =-\gamma\sum_{j=k+1}^m a_j+\rho_{\gamma,k}\\
    &=-A\gamma(m-k)-\gamma(S_m^\eta-S_k^\eta)+\rho_{\gamma,k},\\
    |\rho_{\gamma,k}|&\le\gamma^2\sum_{j=k+1}^m a_j^2\le\gamma^2m\overline a^2.
    \label{eq:transfer-log-expansion}
}
Since $0\le T-\gamma m<\gamma$, subtraction of $\log h_T(k\gamma)=-A(T-k\gamma)$ gives
\eqs*{
    \max_{1\le k\le m}|\log K_{k,m}-\log h_T(k\gamma)|
    &\le A\gamma+2\gamma\max_{0\le k\le m}|S_k^\eta|+\gamma^2m\overline a^2
    \pto 0.
}
Both logarithms are nonpositive, and $|e^x-e^y|\le|x-y|$ for $x,y\le0$. Therefore,
\eqs{D_{\gamma,T}:=\max_{1\le k\le m}|K_{k,m}-h_T(k\gamma)|\pto0.\label{eq:transfer-kernel-uniform}}
The sign of the increments will also be needed. Directly from the product definition,
\eqs{
    K_{k,m}&=(1-\gamma a_{k+1})K_{k+1,m},\\
    K_{k+1,m}-K_{k,m}&=\gamma a_{k+1}K_{k+1,m}\ge0,\\
    \sum_{k=1}^{m-1}(K_{k+1,m}-K_{k,m})&=K_{m,m}-K_{1,m}=1-K_{1,m}\le1.
    \label{eq:transfer-positive-increments}
}

\paragraph{Step 3: Transfer on a fixed window by summation by parts.}
Define the rescaled score partial-sum process by
\eqs{W_\gamma(s):=\frac\gamma{b_\gamma}\sum_{k=1}^{\lfloor s/\gamma\rfloor}\xi_k=\frac1{v(1/\gamma)}\sum_{k=1}^{\lfloor s/\gamma\rfloor}\xi_k.\label{eq:transfer-score-path}}
We first relate its time scale to \eqref{eq:transfer-fclt-assumption}. Let $N_\gamma:=\lfloor1/\gamma\rfloor$, $\lambda_\gamma:=(\gamma N_\gamma)^{-1}$, and $V_N(s):=v(N)^{-1}\sum_{k=1}^{\lfloor Ns\rfloor}\xi_k$. Then
\eqs*{
    W_\gamma(s)&=\frac{v(N_\gamma)}{v(1/\gamma)}V_{N_\gamma}(\lambda_\gamma s),\qquad
    \frac{v(N_\gamma)}{v(1/\gamma)}\to1,\qquad
    \lambda_\gamma\to1.
}
Apply \eqref{eq:transfer-fclt-assumption} on $[0,T+1]$. Because its limit is continuous, a Skorokhod representation gives uniform convergence of $V_{N_\gamma}$ to $Z$ along any sequence $\gamma\downarrow0$. The time change is harmless, since on that representation
\eqs*{
    \sup_{s\le T}|V_{N_\gamma}(\lambda_\gamma s)-Z(s)|
    &\le\sup_{u\le T+1}|V_{N_\gamma}(u)-Z(u)|
    +\sup_{s\le T}|Z(\lambda_\gamma s)-Z(s)|
    \to0.
}
Here the second term tends to zero by uniform continuity of $Z$. It follows that $W_\gamma\Rightarrow Z$ in $D([0,T])$ with the $J_1$ topology.

Consider the weighted sum on this forward block:
\eqs{I_{\gamma,T}:=\frac\gamma{b_\gamma}\sum_{k=1}^mK_{k,m}\xi_k.\label{eq:transfer-window-sums}}
Put $w_k:=W_\gamma(k\gamma)$, so that $w_0=0$ and $w_m=W_\gamma(T)$. Summation by parts is the exact identity
\eqs{
    I_{\gamma,T}
    &=\sum_{k=1}^mK_{k,m}(w_k-w_{k-1})\\
    &=K_{m,m}w_m-K_{1,m}w_0-\sum_{k=1}^{m-1}(K_{k+1,m}-K_{k,m})w_k\\
    &=W_\gamma(T)-\sum_{k=1}^{m-1}(K_{k+1,m}-K_{k,m})W_\gamma(k\gamma).
    \label{eq:transfer-abel-window}
}
Summation by parts moves the random weights from the score increments onto the partial-sum path. The resulting discrete derivative of the kernel is a positive measure of total mass at most one.

To justify the limit of the last sum, define positive measures on $[0,T]$ by
\eqs*{
    \mu_{\gamma,T}&:=\sum_{k=1}^{m-1}(K_{k+1,m}-K_{k,m})\delta_{k\gamma},\qquad
    \mu_T(ds):=Ae^{-A(T-s)}\,ds,
}
where $\delta_x$ is the unit point mass at $x$. Both measures have mass at most one. Write $F_{\gamma,T}(s):=\mu_{\gamma,T}([0,s])$ and $F_T(s):=\mu_T([0,s])$. For $j_\gamma(s):=\min(m,\lfloor s/\gamma\rfloor+1)$, telescoping gives
\eqs*{
    F_{\gamma,T}(s)&=K_{j_\gamma(s),m}-K_{1,m},\qquad F_T(s)=h_T(s)-h_T(0),\\
    |F_{\gamma,T}(s)-F_T(s)|
    &\le|K_{j_\gamma(s),m}-h_T(j_\gamma(s)\gamma)|+|h_T(j_\gamma(s)\gamma)-h_T(s)|\\
    &\quad+|K_{1,m}-h_T(\gamma)|+|h_T(\gamma)-h_T(0)|\\
    &\le2D_{\gamma,T}+2A\gamma.
}
The last inequality uses $|j_\gamma(s)\gamma-s|\le\gamma$ and $\sup_{s\le T}|h_T'(s)|\le A$. In particular,
\eqs{\|F_{\gamma,T}-F_T\|_\infty\pto0.\label{eq:transfer-measure-convergence}}

We record explicitly why this convergence is sufficient for integrating a continuous path. Let $z_\gamma\to z$ uniformly, with $z$ continuous. For a partition $0=t_0<\cdots<t_R=T$ of mesh at most $\delta$, put $z^\delta(s):=z(t_{r-1})$ on $(t_{r-1},t_r]$. Define $\omega_z(\delta):=\sup_{|s-t|\le\delta}|z(s)-z(t)|$. None of the measures has an atom at zero, so
\eqs{
    \left|\int z_\gamma\,d\mu_{\gamma,T}-\int z\,d\mu_T\right|
    &\le\left|\int(z_\gamma-z)\,d\mu_{\gamma,T}\right|
    +\left|\int(z-z^\delta)\,d\mu_{\gamma,T}\right|\\
    &\quad+\left|\int z^\delta\,d(\mu_{\gamma,T}-\mu_T)\right|
    +\left|\int(z^\delta-z)\,d\mu_T\right|\\
    &\le\|z_\gamma-z\|_\infty+2\omega_z(\delta)\\
    &\quad+\sum_{r=1}^R|z(t_{r-1})|
    \left|(F_{\gamma,T}-F_T)(t_r)-(F_{\gamma,T}-F_T)(t_{r-1})\right|\\
    &\le\|z_\gamma-z\|_\infty+2\omega_z(\delta)
    +2R\|z\|_\infty\|F_{\gamma,T}-F_T\|_\infty.
    \label{eq:transfer-integral-continuity}
}
First let $\gamma\downarrow0$ with the partition fixed, and then let $\delta\downarrow0$. The right-hand side tends to zero.

Since the limiting cumulative function in \eqref{eq:transfer-measure-convergence} is deterministic, Slutsky's theorem gives the joint convergence
\eqs*{(W_\gamma,F_{\gamma,T})\Rightarrow(Z,F_T).}
On a Skorokhod representation both coordinates converge uniformly, because their limits are continuous. Apply \eqref{eq:transfer-integral-continuity} on that representation, and use \eqref{eq:transfer-abel-window}. This yields
\eqs{
    I_{\gamma,T}&\Rightarrow J_T,\qquad
    J_T:=Z(T)-A\int_0^T e^{-A(T-s)}Z(s)\,ds
    =\int_0^T e^{-A(T-s)}\,dZ(s).
    \label{eq:transfer-window-limit}
}
The last equality defines the integral by integration by parts, using $Z(0)=0$. This joint argument does not require independence of the score path and the products.

\paragraph{Step 4: A uniform bound for the remote past.}
Write $w_{\gamma,j}:=K_{-j,0}$, $j\ge0$. The product definition and the slope bounds imply, for every realization,
\eqs{
    0<w_{\gamma,j}\le q_\gamma^j,\qquad
    w_{\gamma,j+1}=(1-\gamma a_{-j})w_{\gamma,j},\qquad
    0\le w_{\gamma,j}-w_{\gamma,j+1}=\gamma a_{-j}w_{\gamma,j}\le\gamma\overline a q_\gamma^j.
    \label{eq:transfer-tail-envelopes}
}
Fix an integer $0\le J<n$ and let $L_\gamma:=\lceil1/\gamma\rceil$. Partition $J\le j<n$ into blocks starting at $p_b:=J+bL_\gamma<n$, with lengths $r_b:=\min(L_\gamma,n-p_b)$. Within a block, set
\eqs*{S_{b,r}:=\sum_{i=0}^{r-1}\xi_{-p_b-i},\qquad 1\le r\le L_\gamma,\qquad S_{b,0}:=0.}
The sum can be written in increasing time order. By stationarity its distribution equals that of $\sum_{t=1}^r\xi_t$. Hence \eqref{eq:transfer-block-assumption} gives $\|S_{b,r}\|_2\le Cv(r)$ uniformly in $b$ and $J$.

We also need a bound uniform over the lengths within one block. Potter's bound for the regularly varying function $v$ gives, for some $r_0$ and $r_0\le r\le L_\gamma$,
\eqs*{\frac{v(r)}{v(L_\gamma)}\le C\left(\frac r{L_\gamma}\right)^{\kappa/2}\le C.}
The finitely many $r<r_0$ obey the same bound for small $\gamma$, since $v(L_\gamma)\to\infty$. Thus
\eqs{\sup_{1\le r\le L_\gamma}\|S_{b,r}\|_2\le C v(L_\gamma).\label{eq:transfer-block-norms}}
This is a bound on each partial-sum norm, not a bound on the norm of their maximum.

Apply summation by parts within the block, including a possibly shorter final block, before taking norms:
\eqs*{
    \sum_{i=0}^{r_b-1}w_{\gamma,p_b+i}\xi_{-p_b-i}
    &=\sum_{i=0}^{r_b-1}w_{\gamma,p_b+i}(S_{b,i+1}-S_{b,i})\\
    &=w_{\gamma,p_b+r_b-1}S_{b,r_b}
    +\sum_{i=0}^{r_b-2}(w_{\gamma,p_b+i}-w_{\gamma,p_b+i+1})S_{b,i+1}.
}
Using the deterministic envelopes in \eqref{eq:transfer-tail-envelopes} inside each norm, followed by \eqref{eq:transfer-block-norms}, gives
\eqs{
    \Bigg\|\sum_{i=0}^{r_b-1}w_{\gamma,p_b+i}\xi_{-p_b-i}\Bigg\|_2
    &\le q_\gamma^{p_b+r_b-1}\|S_{b,r_b}\|_2
    +\gamma\overline a\sum_{i=0}^{r_b-2}q_\gamma^{p_b+i}\|S_{b,i+1}\|_2\\
    &\le Cq_\gamma^{p_b}v(L_\gamma)(1+\gamma\overline a L_\gamma).
    \label{eq:transfer-one-block-tail}
}
Here $1\le\gamma L_\gamma<1+\gamma$, and regular variation implies
\eqs*{\frac{\gamma v(L_\gamma)}{b_\gamma}=\frac{v(L_\gamma)}{v(1/\gamma)}\to1.}
Minkowski's inequality and the geometric series therefore yield
\eqs{
    \Bigg\|\frac\gamma{b_\gamma}\sum_{j=J}^{n-1} K_{-j,0}\xi_{-j}\Bigg\|_2
    &\le C\sum_{b=0}^\infty q_\gamma^{J+bL_\gamma}
    =C\frac{q_\gamma^J}{1-q_\gamma^{L_\gamma}}\\
    &\le\frac{C}{1-e^{-\underline a}}q_\gamma^J
    \le C' e^{-\underline a\gamma J}.
    \label{eq:transfer-tail-bound}
}
The denominator bound follows from $q_\gamma^{L_\gamma}\le e^{-\underline a\gamma L_\gamma}\le e^{-\underline a}$. The bound is uniform in $J$ and $n$, and also holds for $J=n$, when the sum is empty.

In particular, taking $J=\lfloor T/\gamma\rfloor$ gives a bound $Ce^{-\underline a T}$ for the tail, after absorbing $e^{\underline a\gamma}$ into the constant.
All bounds used deterministic envelopes for the random weights; no expectation of a product of a weight and a score was factorized.

\paragraph{Step 5: The limit of the finite weighted sum.}
Fix $T>0$ and set $m=\lfloor T/\gamma\rfloor$. Since $n\gamma\to\infty$, we have $m\le n$ for all sufficiently small $\gamma$. Keep the most recent $m$ terms of \eqref{eq:transfer-stationary-filters}:
\eqs*{\widetilde I_{\gamma,T}:=\frac\gamma{b_\gamma}\sum_{t=n-m+1}^nK_{t,n}\xi_t.}
Shifting the entire pair $(a_t,\xi_t)$ by $-(n-m)$ and reindexing $k=t-(n-m)$ gives
\eqs*{\widetilde I_{\gamma,T}\deq\frac\gamma{b_\gamma}\sum_{k=1}^mK_{k,m}\xi_k=I_{\gamma,T}\Rightarrow J_T.}
The convergence follows from \eqref{eq:transfer-window-limit}. For the remaining terms, stationarity and $j=n-t$ give
\eqs*{
    \|\mathcal I_\gamma-\widetilde I_{\gamma,T}\|_2
    &=\Bigg\|\frac\gamma{b_\gamma}\sum_{t=1}^{n-m}K_{t,n}\xi_t\Bigg\|_2
    =\Bigg\|\frac\gamma{b_\gamma}\sum_{j=m}^{n-1}K_{-j,0}\xi_{-j}\Bigg\|_2
    \le Ce^{-\underline a\gamma m}\le C'e^{-\underline a T},
}
by \eqref{eq:transfer-tail-bound}.

To identify the limit, put $\sigma_Z^2:=\E Z(1)^2$. The covariance of $Z$ is
\eqs*{\E[Z(s)Z(t)]=\frac{\sigma_Z^2}{2}(s^{2\kappa}+t^{2\kappa}-|t-s|^{2\kappa}).}
For fixed $T$, let $\widetilde Z_T(s):=Z(T)-Z(T-s)$, $0\le s\le T$. Substituting this covariance formula gives
\eqs*{
    \E[\widetilde Z_T(s)\widetilde Z_T(t)]
    &=\E Z(T)^2-\E[Z(T)Z(T-t)]-\E[Z(T-s)Z(T)] +\E[Z(T-s)Z(T-t)]\\
    &=\frac{\sigma_Z^2}{2}(s^{2\kappa}+t^{2\kappa}-|t-s|^{2\kappa}).
}
The two processes are centered Gaussian with continuous paths, so $\widetilde Z_T\deq Z$ on $[0,T]$. Moreover, integration by parts shows
\eqs*{
    \int_0^T e^{-Au}\,d\widetilde Z_T(u)
    &=e^{-AT}Z(T)+A\int_0^T e^{-Au}(Z(T)-Z(T-u))\,du\\
    &=Z(T)-A\int_0^T e^{-Au}Z(T-u)\,du\\
    &=Z(T)-A\int_0^T e^{-A(T-s)}Z(s)\,ds=J_T.
}
Consequently, $J_T\deq J_T^*$, where
\eqs*{J_T^*:=\int_0^T e^{-As}\,dZ(s)=e^{-AT}Z(T)+A\int_0^T e^{-As}Z(s)\,ds.}
Since $\|Z(s)\|_2=\sigma_Zs^\kappa$, the random variable $J_\infty:=A\int_0^\infty e^{-As}Z(s)\,ds$ exists in $L^2$, and
\eqs{
    \|J_T^*-J_\infty\|_2
    &\le\sigma_Z e^{-AT}T^\kappa+A\sigma_Z\int_T^\infty e^{-As}s^\kappa\,ds
    \to0.
    \label{eq:transfer-limit-integral-tail}
}
This defines $\int_0^\infty e^{-As}\,dZ(s):=J_\infty$. Time reversal has been used only for the limiting Gaussian process, not for the observation sequence.

For a bounded function $\varphi$ with Lipschitz constant at most one,
\eqs*{
    |\E\varphi(\mathcal I_\gamma)-\E\varphi(J_\infty)|
    &\le\|\mathcal I_\gamma-\widetilde I_{\gamma,T}\|_2
    +|\E\varphi(\widetilde I_{\gamma,T})-\E\varphi(J_T)|
    +\|J_T^*-J_\infty\|_2.
}
The middle term tends to zero by \eqref{eq:transfer-window-limit} and stationarity. The first term is at most $Ce^{-\underline a T}$ by Step 4, and the last tends to zero as $T\to\infty$ by \eqref{eq:transfer-limit-integral-tail}. Taking these limits in that order proves $\mathcal I_\gamma\Rightarrow J_\infty$, which is \eqref{eq:transfer-stationary-limit}.
\end{proof}

\subsection{Proof of Theorem~\ref{thm:short-memory-main}}
\label{app:short-memory-main}
\begin{proof}
By Proposition~\ref{prop:nonlinear-linearization}, it suffices to identify the limit of $e_n^{\lin}/\sqrt\gamma$.

\paragraph{Step 1: The score FCLT and the block moment bound.}
The score is centered by the expectile equation and belongs to $L^2$ because $g$ is Lipschitz. The coupling inequality and \eqref{eq:score-fdm-transfer} give
\eqs*{\sum_{r=0}^\infty\|P_0g(-U_r)\|_2\le\sum_{r=0}^\infty\delta_{g,2}(r)=\Theta_{g,2}\le\overline\tau\Theta_{X,2}<\infty.}
Thus \citet[Theorem 3(ii)]{Wu2011Asymptotic} applies to the score process, with long-run variance $\Omega_g$ from \eqref{eq:score-lrv}. In the equivalent step-partial-sum formulation, for every fixed $T>0$,
\eqs{\frac1{\sqrt N}\sum_{t=1}^{\lfloor Ns\rfloor}g(-U_t)\Rightarrow\sqrt{\Omega_g}W(s)\quad\text{in }D([0,T]),\ J_1,\label{eq:short-memory-score-fclt}}
where $W$ is standard Brownian motion. The covariance bound in Section~\ref{subsec:dependence} also gives
\eqs{
    \left\|\sum_{t=1}^m g(-U_t)\right\|_2^2
    &=\sum_{s,t=1}^m\Cov(g(-U_s),g(-U_t))\\
    &\le m\sum_{h\in\mathbb Z}|\Cov(g(-U_0),g(-U_h))|\le m\Theta_{g,2}^2.
    \label{eq:short-memory-block-bound}
}

\paragraph{Step 2: Transfer to the finite weighted sum.}
The pair $(g'(-U_t),g(-U_t))$ is jointly stationary, and $g'(-U_t)\in[\underline\tau,\overline\tau]$. The slope process is a Bernoulli shift and hence ergodic. Together with Step 1, these facts verify the hypotheses of Lemma~\ref{lem:correlated-transfer} with
\eqs*{a_t=g'(-U_t),\qquad\xi_t=g(-U_t),\qquad v(N)=\sqrt N,\qquad b_\gamma=\sqrt\gamma.}
Since \eqref{eq:burn-in-main} implies $n\gamma\to\infty$, the lemma yields
\eqs{\sqrt\gamma\sum_{t=1}^nK_{t,n}g(-U_t)\Rightarrow\sqrt{\Omega_g}\int_0^\infty e^{-As}\,dW(s)\sim\N\Big(0,\frac{\Omega_g}{2A}\Big),\label{eq:short-memory-random-filter-limit}}
since $\Omega_g\int_0^\infty e^{-2As}\,ds=\Omega_g/(2A)$.

\paragraph{Step 3: Initialization and the nonlinear remainder.}
The initialization term in \eqref{eq:linearized-solution} vanishes by \eqref{eq:burn-in-main}:
\eqs*{\frac{|K_{0,n}e_0|}{\sqrt\gamma}\le\frac{|e_0|(1-\gamma\underline\tau)^n}{\sqrt\gamma}\le |e_0|\gamma^{\nu\underline\tau-1/2}\to0.}
Thus \eqref{eq:linearized-solution} and the symmetry of the centered Gaussian limit give $e_n^{\lin}/\sqrt\gamma\Rightarrow\N(0,\Omega_g/(2A))$. Finally, Proposition~\ref{prop:nonlinear-linearization}, which combines geometric forgetting, stationary localization, and the conditional nonlinear-remainder bound, gives
\eqs*{\frac{\E|e_n-e_n^{\lin}|}{\sqrt\gamma}\le C\sqrt\gamma+C\frac{(1-\gamma\underline\tau)^{n-1}}{\sqrt\gamma}\to0.}
Slutsky's theorem completes the proof:
\eqs*{\frac{e_n}{\sqrt\gamma}=\frac{e_n^{\lin}}{\sqrt\gamma}+o_P(1)\Rightarrow\N\Big(0,\frac{\Omega_g}{2A}\Big).}
\end{proof}

\paragraph{Burn-in and random initialization.}\label{app:burn-in-calculations}
Condition~\eqref{eq:burn-in-main} makes the initialization negligible at scale $\sqrt\gamma$. Indeed,
\[
    (1-\gamma\underline\tau)^n
    \le e^{-\gamma\underline\tau n}
    \le \gamma^{\nu\underline\tau}
    =o(\sqrt\gamma)
\]
whenever $\nu>1/(2\underline\tau)$.

The same limit holds for an $\F_0$-measurable initialization $\theta_0\in L^2$ that does not depend on $\gamma$. Indeed, coupling this recursion with one initialized at a deterministic $c$ gives
\begingroup
\[
\frac{|\theta_n^{[\theta_0]}-\theta_n^{[c]}|}{\sqrt\gamma}\le\frac{(1-\gamma\underline\tau)^n}{\sqrt\gamma}|\theta_0-c|\longpto0,
\]
\endgroup
where $\theta_n^{[u]}$ denotes the iterate initialized at $u$. This follows from \eqref{eq:burn-in-main} and requires no independence between $\theta_0$ and the observations.

\section{Proof for Long-Memory Processes}\label{app:proof_long_memory}
The unit-variance normalization in Assumption~\ref{ass:long-memory} is without loss of generality. For innovations with $\sigma_\varepsilon:=\|\varepsilon_0\|_2>0$, replacing $(\varepsilon_t,c_j)$ by $(\varepsilon_t/\sigma_\varepsilon,\sigma_\varepsilon c_j)$ leaves $X_t$ unchanged, gives $\E[(\varepsilon_t/\sigma_\varepsilon)^2]=1$, and replaces $\ell$ by $\sigma_\varepsilon\ell$.

\subsection{Proof of Lemma~\ref{lem:long-memory-stationary-localization}}
\label{app:long-memory-localization}

For \eqref{eq:long-memory-covariance}, the asymptotic follows from the regularly varying convolution,
\begingroup
\[
\frac{\sum_{j=0}^\infty c_jc_{j+h}}{h^{1-2\beta}\ell(h)^2}\longrightarrow\int_0^\infty x^{-\beta}(1+x)^{-\beta}\,dx\in(0,\infty).
\]
\endgroup

\begin{proof}
We follow the argument of Lemma~\ref{lem:stationary-localization}.
Under Assumption~\ref{ass:long-memory},
\[
    \delta_{U,2}(k) = \|U_k-U_k^{\{0\}}\|_2=|c_k| \|\varepsilon_0-\varepsilon_0'\|_2,
    \qquad k\ge0.
\]
Hence, following \eqref{eq:e-fdm-propagation},
\[
    \delta_{e,2}(k) \le \gamma\overline\tau\|\varepsilon_0-\varepsilon_0'\|_2\sum_{j=0}^{k}(1-\gamma\underline\tau)^{k-j}|c_j|.
\]
This yields
\begin{equation}\label{eq:long-memory-localization-convolution}
\begin{aligned}
    \Var(e_0^\circ)
    &\le\sum_{k=0}^{\infty}\delta_{e,2}^2(k)\le\gamma^2\overline\tau^2\|\varepsilon_0-\varepsilon_0'\|_2^2 \Lambda_\gamma^2,\qquad
    \Lambda_\gamma:=\Bigg(\sum_{k=0}^{\infty}\bigg(\sum_{j=0}^{k}(1-\gamma\underline\tau)^{k-j}|c_j|\bigg)^2\Bigg)^{1/2}.
\end{aligned}
\end{equation}

It remains to bound the convolution in $\Lambda_\gamma$.  Set $N:=\lfloor\gamma^{-1}\rfloor$.  Splitting the coefficient sequence at $N$ and applying Young's convolution inequality for each term gives
\[
    \Lambda_\gamma\le\Bigg(\sum_{j=0}^{\infty}(1-\gamma\underline\tau)^{2j}\Bigg)^{1/2}\sum_{j=0}^{N}|c_j|+\Bigg(\sum_{j=0}^{\infty}(1-\gamma\underline\tau)^j\Bigg)\Bigg(\sum_{j>N}c_j^2\Bigg)^{1/2}.
\]
Summing the geometric series gives
\[
    \Bigg(\sum_{j=0}^{\infty}(1-\gamma\underline\tau)^{2j}\Bigg)^{1/2}\lesssim\gamma^{-1/2},
    \qquad
    \sum_{j=0}^{\infty}(1-\gamma\underline\tau)^j\lesssim\gamma^{-1}.
\]
Using Karamata's theorem (see \citealp[Lemma 5]{Wu2003Empirical}) gives
\[
    \sum_{j=0}^{N}|c_j|\lesssim N^{1-\beta}\ell(N),
    \qquad
    \Bigg(\sum_{j>N}c_j^2\Bigg)^{1/2}\lesssim N^{1/2-\beta}\ell(N).
\]
Hence, as $N\asymp\gamma^{-1}$,
\[
    \Lambda_\gamma
    \lesssim\gamma^{-1/2}N^{1-\beta}\ell(N)
      +\gamma^{-1}N^{1/2-\beta}\ell(N)
    \lesssim\gamma^{\beta-3/2}\ell(1/\gamma)=s_\gamma/\gamma.
\]
Substitution into \eqref{eq:long-memory-localization-convolution} yields
\begin{equation}\label{eq:long-memory-localization-variance}
    \Var(e_0^\circ)\lesssim s_\gamma^2.
\end{equation}

Combining \eqref{eq:stationary-e-mean} in Lemma~\ref{lem:stationary-localization} and \eqref{eq:long-memory-localization-variance}, we obtain
\[
    \E[|e_0^\circ|^2]
    =\Var(e_0^\circ)+|\E[e_0^\circ]|^2
    \le\Big(1+\frac{\overline\tau^2}{\underline\tau^2}\Big)
       \Var(e_0^\circ)
    \lesssim s_\gamma^2.
\]
Stationarity gives the same bound for every $t$.
\end{proof}

\subsection{Long-memory functional CLT}
\label{app:long-memory-fclt}

Throughout Appendix~\ref{app:proof_long_memory} we write $B_H$ for the rank-one Hermite process of \citet{wu2006unit}, a fractional Brownian motion with Hurst index $H=3/2-\beta$, and $W_\gamma(s)$ for the rescaled score partial-sum process of \eqref{eq:long-memory-score-fclt}.

\begin{proposition}\label{prop:long-memory-FCLT}
Suppose Assumption~\ref{ass:long-memory} holds, and let $B_H$ denote the fractional Brownian motion defined in \eqref{eq:fBM}.
Then for every fixed $T>0$,
\begin{equation}\label{eq:long-memory-score-fclt}
\begin{aligned}
\frac{\gamma}{s_\gamma} \sum_{k=1}^{\lfloor s/\gamma\rfloor} g(-U_k)
\Rightarrow -A B_H(s)
\end{aligned}
\end{equation}
holds in $D([0,T])$ with the $J_1$ topology, as $\gamma\downarrow0$.
Moreover $\big\|\sum_{t=1}^m g(-U_t)\big\|_2\lesssim m^{3/2-\beta}\ell(m)$ for all
$m\geq1$.
\end{proposition}

\begin{proof}
We apply \citet[Theorem 1]{wu2006unit} to $K(x) := g(\theta_\tau^*-x)$ so that $g(-U_t) = K(X_t)$.
Following the notation of the paper, define
\[
\begin{aligned}
    x_{n,0} &:= \sum_{j=n}^{\infty}c_j\varepsilon_{n-j},
    \qquad x_{n,1}:=x_{n,0}+c_{n-1}\varepsilon_1,\\
    K_n(w) &:= \E[K(w+X_n-x_{n,0})],
    \qquad K_\infty(w):=\E[K(w+X_0)].
\end{aligned}
\]
For a function $h$, write $h(x;\lambda):=\sup_{|y|\le\lambda}|h(x+y)|$.

We first identify the power rank of $K$. 
By the expectile equation, $K_\infty(0)=\E[g(\theta_\tau^*-X_0)]=0$
whereas $K_\infty'(0) = -\E[g'(\theta_\tau^*-X_0)] = -A\neq 0$.
Hence $K$ has power rank $p=1$ \citep{Ho1997Limit,wu2006unit}.

We next verify Condition~1 of \citet{wu2006unit} with $p=1$ and $q=4$. Since $\E|\varepsilon_1|^6<\infty$, the required fourth-moment condition is immediate. For $p=1$, Condition~1 reduces to showing that, for some $\lambda>0$,
$$
\sum_{r=0}^{2}\|K_{n-1}^{(r)}(x_{n, 0} ; \lambda)\| + \||\varepsilon_1|^{2} K_{n-1} (x_{n, 1})\| +\|\varepsilon_1 K_{n-1}^{(1)}(x_{n, 1})\|=O(1),
$$
where $K_{n-1}^{(r)}$ is $r$-th derivative of $K_{n-1}$.
We then show that each term is $O(1)$.

\textbf{Term 1.}
Let $F_n$ and $f_n$ be the CDF and density of $X_n-x_{n,0}$. Since this finite sum contains the independent component $c_0\varepsilon_n$, Assumption~\ref{ass:long-memory}\ref{ass:lm-innovation} gives $\sup_{n\ge1}\|f_n\|_\infty <\infty$.

Then the first and second derivatives of $K_n(w)=\E\big[g(\theta_\tau^*-w-(X_n-x_{n,0}))\big]$ are
\[
\begin{aligned}
    K_n'(w) &=-\E\big[g'(\theta_\tau^*-w-(X_n-x_{n,0}))\big]
    = - \tau - (1 - 2\tau) F_n (\theta_\tau^*-w),\\
    K_n''(w) &= \frac{d}{dw}\big[-\tau-(1-2\tau)F_n(\theta_\tau^*-w)\big]
    =(1-2\tau)f_n(\theta_\tau^*-w).
\end{aligned}
\]
Thus $K_n\in C^2(\mathbb R)$ for every $n$, $\sup_n \|K_n'\|_\infty + \sup_n \|K_n''\|_\infty < \infty$ and
\begin{align*}
K_{n-1}'(x_{n, 0} ; \lambda) & =\sup _{|y| \leq \lambda}|K_{n-1}'(x_{n, 0}+y)| 
\leq \sup _{w \in \mathbb{R}}|K_{n-1}'(w)| = \overline \tau = O(1), \\
K_{n-1}''(x_{n, 0} ; \lambda) &\leq \sup_{w \in\R} |K_{n-1}''(w)| = |1 - 2\tau| \|f_{n-1}\|_\infty = O(1).
\end{align*}
Since $g$ is $\overline\tau$-Lipschitz and Assumption~\ref{ass:long-memory}\ref{ass:lm-innovation} gives $\sup_n\|X_n-x_{n,0}\|_2<\infty$,
\begin{align*}
\|K_{n-1}(x_{n,0};\lambda)\|_2^2
&= \E\left[ \sup_{|y|\le\lambda} |K_{n-1}(x_{n,0}+y)|^2 \right] \\
&\le \overline\tau^2 \E\left[ \sup_{|y|\le\lambda} \bigl( |\theta_\tau^*| +|x_{n,0}+y| +\E|X_{n-1} - x_{n-1,0}| \bigr)^2 \right] \\
&\lesssim \E\left(1+|x_{n,0}|+\lambda\right)^2 \\
&\lesssim 1+\|x_{n,0}\|_2^2+\lambda^2 
= 1+\E[\varepsilon_1^2]\sum_{j=n}^{\infty}c_j^2+\lambda^2 \\
&\le 1+\E[\varepsilon_1^2]\sum_{j=0}^{\infty}c_j^2+\lambda^2 =O(1).
\end{align*}
As such, for every fixed $\lambda>0$, $\sum_{r=0}^{2}\|K_{n-1}^{(r)}(x_{n,0};\lambda)\|_2 = O(1)$.

\textbf{Term 2.}
Using $x_{n,1}=x_{n,0}+c_{n-1}\varepsilon_1$ and the independence of $x_{n,0}$ and $\varepsilon_1$,
\[
\begin{aligned}
    \|\varepsilon_1^2K_{n-1}(x_{n,1})\|_2
    &\lesssim \|\varepsilon_1^2(1+|x_{n,0}|+|c_{n-1}\varepsilon_1|)\|_2\\
    &\lesssim \|\varepsilon_1\|_4^2+\|\varepsilon_1\|_4^2\|x_{n,0}\|_2+|c_{n-1}|\|\varepsilon_1\|_6^3 =O(1).
\end{aligned}
\]

\textbf{Term 3.}
Since $|K_{n-1}' (x_{n, 1})| \leq \overline\tau$, Assumption~\ref{ass:long-memory}\ref{ass:lm-innovation} implies
\[
    \|\varepsilon_1 K_{n-1}' (x_{n, 1})\|_2  \leq \overline\tau \|\varepsilon_1\|_2 = O(1),
\]

Therefore, all terms in Condition 1 are uniformly bounded, and \citet[Theorem 1(i)]{wu2006unit} yields functional convergence of the linearly interpolated partial sums in $C([0,T])$ with the uniform topology, after rescaling time from $[0,1]$ to $[0,T]$.
Let $H=3/2-\beta$, write $W_m(s):=m^{-H}\ell(m)^{-1}\sum_{k=1}^{\lfloor ms\rfloor}g(-U_k)$, and let $\widetilde W_m$ be its linear interpolation on the mesh $k/m$. For every $\epsilon>0$,
\[
\begin{aligned}
\mathbb P\!\left(\sup_{0\le s\le T}|\widetilde W_m(s)-W_m(s)|>\epsilon\right)
&\le \mathbb P\!\left(\max_{1\le k\le\lfloor mT\rfloor+1}|g(-U_k)|>\epsilon m^H\ell(m)\right)\\
&\le \frac{(\lfloor mT\rfloor+1)\E[g(-U_0)^2]}{\epsilon^2m^{2H}\ell(m)^2}\longrightarrow0.
\end{aligned}
\]
Here we used stationarity, $g(-U_0)\in L^2$, $H>1/2$, and slow variation of $\ell$. Hence the step process has the same limit in $D([0,T])$ with the $J_1$ topology,
\[
    \Bigg\{\frac{1}{m^{3/2-\beta} \ell (m)}\sum_{k=1}^{\lfloor ms\rfloor}g(-U_k) :0\le s\le T\Bigg\}
    \Rightarrow \{-AB_H(s):0\le s\le T\},
\]
where $\mathbb{B}$ is a standard two-sided Brownian motion and $B_H$ is the rank-one Hermite process of \citet{wu2006unit},
\[
    B_H(t) := \int_{-\infty}^{t}\!\int_{0}^{t}(v-u)_+^{-\beta}\,dv\,d\mathbb{B}(u)
        = \frac{1}{1-\beta}\int_{-\infty}^{t} \Big((t-u)^{1-\beta}-(-u)_+^{1-\beta}\Big)\,d\mathbb{B}(u),
\]
a fractional Brownian motion with Hurst index $H = 3/2-\beta$, for which
\[
    \sigma_H^2 := \Var B_H(1) = \frac{1}{(1-\beta)^{2}}\int_{\mathbb{R}} \Big((1-u)_+^{1-\beta}-(-u)_+^{1-\beta}\Big)^{2}du .
\]

Taking $m=\lfloor1/\gamma\rfloor$ gives \eqref{eq:long-memory-score-fclt}.

It remains to record the partial-sum bound used below. The rank-one reduction in the proof of \citet[Theorem 1]{wu2006unit} gives
\[
    \sum_{t=1}^m g(-U_t) =-A\sum_{t=1}^mX_t+R_m,
    \qquad \|R_m\|_2=o(m^{3/2-\beta}\ell(m)).
\]
Also, Karamata's theorem gives $|\Cov(X_0,X_h)|\lesssim h^{1-2\beta}\ell(h)^2$ so that
\[
\begin{aligned}
    \Bigg\|\sum_{t=1}^mX_t\Bigg\|_2^2
    &\le m\Var(X_0)+2\sum_{h=1}^{m-1}(m-h)|\Cov(X_0,X_h)|\\
    &\lesssim m^{3-2\beta}\ell(m)^2.
\end{aligned}
\]
Combining the last two displays proves the result.
\end{proof}

\subsection{Proof of Theorem~\ref{thm:long-memory-last-iterate}}
\label{app:long-memory-last-iterate}

Recall the centered recursion \eqref{eq:centered-recursion}, $e_t=e_{t-1}-\gamma g(e_{t-1}-U_t)$ with $e_t=\theta_t-\theta_\tau^*$.

\begin{proof}
We first verify the sufficient burn-in condition. If $n\gamma\ge\nu\log(\gamma^{-1})$ with $\nu>(\beta-1/2)/\underline\tau$, then
\[
\frac{(1-\gamma\underline\tau)^n}{s_\gamma}\le\frac{e^{-n\gamma\underline\tau}}{\gamma^{\beta-1/2}\ell(1/\gamma)}\le\frac{\gamma^{\nu\underline\tau-(\beta-1/2)}}{\ell(1/\gamma)}\to0.
\]
The last step uses $\nu\underline\tau-(\beta-1/2)>0$ and the fact that $x^{-\delta}/\ell(x)\to0$ for every $\delta>0$ when $\ell$ is positive and slowly varying. This proves the sufficiency stated in Theorem~\ref{thm:long-memory-last-iterate}.

By Lemma~\ref{lem:long-memory-stationary-localization}, the stationary solution
$e_t^\circ=\theta_t^\circ(\gamma)-\theta_\tau^*$ of Proposition~\ref{prop:stationary-solution}
satisfies
\begin{equation}\label{eq:lm-localization-recalled}
    \E[|e_0^\circ|^2]\lesssim s_\gamma^2.
\end{equation}
Coupling the forward and stationary recursions with the same observations, the geometric forgetting bound \eqref{eq:GMC} therefore gives
\begin{equation}\label{eq:lm-forward-moment}
    \|e_t\|_2
    \lesssim s_\gamma+(1-\gamma\underline\tau)^t.
\end{equation}

Next, define the linearized recursion
\[
    e_t^{\lin} =(1-\gamma g'(-U_t))e_{t-1}^{\lin}-\gamma g(-U_t),
    \qquad e_0^{\lin}=e_0.
\]
The exact remainder is
\[
    R_t :=g(e_{t-1}-U_t)-g(-U_t)-g'(-U_t)e_{t-1}.
\]
The crossing bound derived in Appendix~\ref{app:nonlinear-linearization} gives
\[
    |R_t| \le |1-2\tau|\,|e_{t-1}|\,\1\big(\min(0,e_{t-1})\le U_t\le\max(0,e_{t-1})\big).
\]
The conditional density bound of Section~\ref{subsec:long-memory-model} gives
\[
    \E[|R_t|\mid\F_{t-1}] \lesssim |e_{t-1}|^2.
\]
Solving the difference between the two recursions and using the contraction bound,
\[
\begin{aligned}
    \E[|e_n-e_n^{\lin}|]
    &\le \gamma\sum_{t=1}^n(1-\gamma\underline\tau)^{n-t}\E[|R_t|]\\
    &\lesssim s_\gamma^2+(1-\gamma\underline\tau)^n
    =o(s_\gamma).
\end{aligned}
\]
Hence it remains to analyze $e_n^{\lin}$.
Set $a_t=g'(-U_t)$ and $\xi_t=g(-U_t)$. The pair is jointly stationary, $\E\xi_0=0$, and $a_t\in[\underline\tau,\overline\tau]$. The slope process is a Bernoulli shift and hence ergodic. Proposition~\ref{prop:long-memory-FCLT} verifies the remaining hypotheses of Lemma~\ref{lem:correlated-transfer} with
\eqs*{v(N)=N^H\ell(N),\qquad b_\gamma=\gamma^{1-H}\ell(1/\gamma)=s_\gamma,\qquad Z=-AB_H.}
Since $s_\gamma\to0$ and $-\log(1-\gamma\underline\tau)\sim\gamma\underline\tau$, the burn-in condition \eqref{eq:long-memory-burn-in} implies $n\gamma\to\infty$. The lemma therefore gives
\eqs*{\frac\gamma{s_\gamma}\sum_{t=1}^nK_{t,n}g(-U_t)\Rightarrow-A\int_0^\infty e^{-As}\,dB_H(s).}
The initialization term vanishes by the same burn-in condition:
\eqs*{\frac{|K_{0,n}e_0|}{s_\gamma}\le\frac{|e_0|(1-\gamma\underline\tau)^n}{s_\gamma}\to0.}
Solving the linearized recursion and combining these two limits yields
\eqs*{\frac{e_n^{\lin}}{s_\gamma}=\frac{K_{0,n}e_0}{s_\gamma}-\frac\gamma{s_\gamma}\sum_{t=1}^nK_{t,n}g(-U_t)\Rightarrow A\int_0^\infty e^{-As}\,dB_H(s).}
Combining this with
$e_n-e_n^{\lin}=o_P(s_\gamma)$ gives the same limit for $e_n/s_\gamma$.

It remains to compute the variance of the centered Gaussian limit $Z:=A\int_0^\infty e^{-As}\,dB_H(s)$. Since $\Var(B_H(1))=\sigma_H^2$, the fBm covariance and its mixed derivative are
\[
\begin{aligned}
    \Cov(B_H(s),B_H(t))&=\frac{\sigma_H^2}{2}\big(s^{2H}+t^{2H}-|s-t|^{2H}\big),\\
    \frac{\partial^2}{\partial s\,\partial t}\Cov(B_H(s),B_H(t))&=\sigma_H^2 H(2H-1)|s-t|^{2H-2},\qquad s\ne t.
\end{aligned}
\]
The covariance between two fBm increments is the integral of this mixed derivative over the corresponding pair of time intervals. Thus, taking the variance of the weighted Riemann sums defining $Z$ gives the double integral below. The singularity at $s=t$ is integrable because $2H-2>-1$, and the exponential weights ensure integrability at infinity. By symmetry, followed by the changes of variables $r=s-t$ on $s>t$ and $u=Ar$, we obtain
\[
\begin{aligned}
    \Var(Z)
    &=\sigma_H^2 H(2H-1)A^2\int_0^\infty\!\int_0^\infty e^{-A(s+t)}|s-t|^{2H-2}\,ds\,dt\\
    &=2\sigma_H^2 H(2H-1)A^2\int_0^\infty\!\int_t^\infty e^{-A(s+t)}(s-t)^{2H-2}\,ds\,dt\\
    &=2\sigma_H^2 H(2H-1)A^2\int_0^\infty\!\int_0^\infty e^{-A(r+2t)}r^{2H-2}\,dr\,dt\\
    &=2\sigma_H^2 H(2H-1)A^2\left(\int_0^\infty e^{-2At}\,dt\right)\left(\int_0^\infty e^{-Ar}r^{2H-2}\,dr\right)\\
    &=2\sigma_H^2 H(2H-1)A^2\frac{1}{2A}\frac{1}{A^{2H-1}}\int_0^\infty e^{-u}u^{2H-2}\,du\\
    &=\sigma_H^2 H(2H-1)A^{2-2H}\Gamma(2H-1)\\
    &=\sigma_H^2 H\Gamma(2H)A^{2-2H}.
\end{aligned}
\]
The last two equalities use the defining gamma integral and $(2H-1)\Gamma(2H-1)=\Gamma(2H)$. This is the variance stated in Theorem~\ref{thm:long-memory-last-iterate}, completing the proof.
\end{proof}

\subsection{Proofs for joint long-memory limits}
\label{app:joint-limit}

Use the level-specific notation in Section~\ref{subsec:joint-long-memory}, and define
\eqs*{\xi_t^{(i)}:=g_i(\theta_i^*-X_t),\qquad a_t^{(i)}:=g_i'(\theta_i^*-X_t).}
We retain $H=3/2-\beta$, $s_\gamma=\gamma^{1-H}\ell(1/\gamma)$, and the normalization of $B_H$ in \eqref{eq:fBM}.
For the proofs below, write $Y_i:=A_i\int_0^\infty e^{-A_i s}\,dB_H(s)$ for the $i$th limiting coordinate.

\subsubsection{Proof of Theorem~\ref{thm:joint-long-memory}}

\begin{proof}
We first identify the common score limit and then apply the proof of Lemma~\ref{lem:correlated-transfer} jointly to the fixed coordinates.

\paragraph{Step 1: A common score limit.}
Fix $T>0$, put $v(N):=N^H\ell(N)$, and define the normalized score partial sums
\eqs*{S_N^{(i)}(s):=\frac1{A_i v(N)}\sum_{t=1}^{\lfloor Ns\rfloor}\xi_t^{(i)},\qquad 0\le s\le T.}
Proposition~\ref{prop:long-memory-FCLT}, with $\gamma=1/N$, gives $S_N^{(i)}\Rightarrow-B_H$ in $D([0,T])$ with $J_1$ for every $i$.

To identify the dependence between these limits, fix $i,j$ and consider the average of their normalized scores,
\eqs*{K_{ij}(x):=\frac12\left(\frac{g_i(\theta_i^*-x)}{A_i}+\frac{g_j(\theta_j^*-x)}{A_j}\right).}
Smoothing and differentiation are linear, so the $C^2$ property and each moment bound in Condition 1 of \citet{wu2006unit} follow by the triangle inequality from the checks in Proposition~\ref{prop:long-memory-FCLT}. Moreover,
\eqs*{
    \E[K_{ij}(X_0)]&=\frac12\left(\frac{\E\xi_0^{(i)}}{A_i}+\frac{\E\xi_0^{(j)}}{A_j}\right)=0,\\
    \left.\frac{d}{dw}\E[K_{ij}(w+X_0)]\right|_{w=0}&=-\frac12\left(\frac{A_i}{A_i}+\frac{A_j}{A_j}\right)=-1.
}
Thus the same application of \citet[Theorem 1(i)]{wu2006unit}, with $p=1$ and $q=4$, yields
\eqs*{\frac{S_N^{(i)}+S_N^{(j)}}2=\frac1{v(N)}\sum_{t=1}^{\lfloor N\,\cdot\rfloor}K_{ij}(X_t)\Rightarrow-B_H\quad\text{in }D([0,T]),\ J_1.}

The marginal FCLTs imply tightness of $(S_N^{(i)})_{i=1}^k$ in the finite product space $D([0,T])^k$. Let $(Z_i)_{i=1}^k$ be any subsequential limit. Each $Z_i$ has the law of $-B_H$ and continuous paths. Averaging is continuous at these paths, so $(Z_i+Z_j)/2$ also has the law of $-B_H$. At every fixed $s\in[0,T]$, these three limiting variables have mean zero and variance $\sigma_H^2s^{2H}$. Hence
\eqs*{
    \Var(Z_i(s)-Z_j(s))
    &=2\Var(Z_i(s))+2\Var(Z_j(s))-4\Var\left(\frac{Z_i(s)+Z_j(s)}2\right)\\
    &=(2+2-4)\sigma_H^2s^{2H}=0.
}
Thus $Z_i(s)=Z_j(s)$ almost surely. Taking a countable dense set of times and using continuity gives $Z_i=Z_j$ on $[0,T]$ almost surely, simultaneously for the finitely many pairs. Every subsequential limit therefore has the law of the diagonal process $(-B_H,\ldots,-B_H)$. Multiplying coordinate $i$ by $A_i$ proves
\eqs{\left(\frac1{v(N)}\sum_{t=1}^{\lfloor Ns\rfloor}\xi_t^{(i)}\right)_{i=1}^k\Rightarrow(-A_iB_H(s))_{i=1}^k\quad\text{in }D([0,T])^k,\ J_1.\label{eq:joint-score-fclt}}

\paragraph{Step 2: Joint transfer to the finite iterates.}
Define $K_{r,t}^{(i)}:=\prod_{j=r+1}^t(1-\gamma a_j^{(i)})$, with $K_{t,t}^{(i)}=1$. Each slope process is bounded and ergodic, as in Appendix~\ref{app:long-memory-last-iterate}. For fixed $T>0$ and $m:=\lfloor T/\gamma\rfloor$, \eqref{eq:transfer-kernel-uniform} and finite $k$ give
\eqs*{\max_{1\le i\le k}\max_{1\le t\le m}|K_{t,m}^{(i)}-e^{-A_i(T-t\gamma)}|\pto0.}
Let $W_\gamma^{(i)}(s):=(\gamma/s_\gamma)\sum_{t=1}^{\lfloor s/\gamma\rfloor}\xi_t^{(i)}$. Since $\gamma/s_\gamma=1/v(1/\gamma)$, \eqref{eq:joint-score-fclt} and the time rescaling in Step 3 of Appendix~\ref{app:transfer} give $(W_\gamma^{(i)})_{i=1}^k\Rightarrow(-A_iB_H)_{i=1}^k$. On the fixed window, summation by parts gives
\eqs*{
    I_{\gamma,T}^{(i)}&:=\frac\gamma{s_\gamma}\sum_{t=1}^mK_{t,m}^{(i)}\xi_t^{(i)}\\
    &=W_\gamma^{(i)}(T)-\sum_{t=1}^{m-1}(K_{t+1,m}^{(i)}-K_{t,m}^{(i)})W_\gamma^{(i)}(t\gamma).
}
The positive measures defined by these weight increments converge to $A_i e^{-A_i(T-s)}\,ds$, by \eqref{eq:transfer-measure-convergence}. Their limits are deterministic, so they converge jointly with the vector score path. Applying \eqref{eq:transfer-integral-continuity} to each coordinate on one common Skorokhod representation yields
\eqs*{(I_{\gamma,T}^{(i)})_{i=1}^k\Rightarrow\left(-A_i\int_0^T e^{-A_i(T-s)}\,dB_H(s)\right)_{i=1}^k.}

Condition~\eqref{eq:joint-long-memory-burn-in} implies $n\gamma\to\infty$, hence $m\le n$ for small $\gamma$. Joint stationarity shifts the most recent $m$ terms of all coordinates together:
\eqs*{\left(\frac\gamma{s_\gamma}\sum_{t=n-m+1}^nK_{t,n}^{(i)}\xi_t^{(i)}\right)_{i=1}^k\deq(I_{\gamma,T}^{(i)})_{i=1}^k.}
For the remaining terms, \eqref{eq:transfer-tail-bound} gives
\eqs*{\max_{1\le i\le k}\left\|\frac\gamma{s_\gamma}\sum_{t=1}^{n-m}K_{t,n}^{(i)}\xi_t^{(i)}\right\|_2\le Ce^{-\gamma m\min_i\underline\tau_i}\le C'e^{-T\min_i\underline\tau_i}.}
Time reversal of the same fractional Brownian motion, as in Step 5 of Appendix~\ref{app:transfer}, gives the joint identity
\eqs*{\left(\int_0^T e^{-A_i(T-s)}\,dB_H(s)\right)_{i=1}^k\deq\left(\int_0^T e^{-A_i s}\,dB_H(s)\right)_{i=1}^k.}
Each integral on the right converges in $L^2$ as $T\to\infty$ by \eqref{eq:transfer-limit-integral-tail}. Taking first $\gamma\downarrow0$ and then $T\to\infty$, using the finite number of tail bounds, proves
\eqs*{\left(\frac\gamma{s_\gamma}\sum_{t=1}^nK_{t,n}^{(i)}\xi_t^{(i)}\right)_{i=1}^k\Rightarrow(-Y_i)_{i=1}^k.}

Write $e_t^{(i)}:=\theta_t^{(i)}-\theta_i^*$, and let $e_t^{\lin,(i)}$ be the corresponding linearization with initial value $e_0^{(i)}$. Its exact solution is
\eqs*{\frac{e_n^{\lin,(i)}}{s_\gamma}=\frac{K_{0,n}^{(i)}e_0^{(i)}}{s_\gamma}-\frac\gamma{s_\gamma}\sum_{t=1}^nK_{t,n}^{(i)}\xi_t^{(i)}.}
The initialization and nonlinear-remainder bounds in Appendix~\ref{app:long-memory-last-iterate} hold at each of the fixed levels. In particular,
\eqs*{\max_{1\le i\le k}\frac{|K_{0,n}^{(i)}e_0^{(i)}|+\E|e_n^{(i)}-e_n^{\lin,(i)}|}{s_\gamma}\le C\left(s_\gamma+\max_{1\le i\le k}\frac{(1-\gamma\underline\tau_i)^n}{s_\gamma}\right)\to0.}
Markov's inequality and a finite union bound make all nonlinear remainders negligible jointly. The preceding displays and Slutsky's theorem prove \eqref{eq:joint-long-memory-limit}.
\end{proof}

\subsubsection{Proof of Corollary~\ref{cor:joint-long-memory-covariance}}

\begin{proof}
\textbf{Step 1: Covariance calculation.}
The fractional-Brownian covariance calculation in Appendix~\ref{app:long-memory-last-iterate} gives
\eqs*{\Cov(Y_i,Y_j)=\sigma_H^2H(2H-1)A_iA_j\int_0^\infty\!\int_0^\infty e^{-A_i s-A_j t}|s-t|^{2H-2}\,ds\,dt.}
For $a,b>0$, separate $s>t$ and $t>s$ and put $u=|s-t|$:
\eqs*{
    \int_0^\infty\!\int_0^\infty e^{-as-bt}|s-t|^{2H-2}\,ds\,dt
    &=\int_0^\infty\!\int_0^\infty e^{-(a+b)t-au}u^{2H-2}\,du\,dt\\
    &\quad+\int_0^\infty\!\int_0^\infty e^{-(a+b)s-bu}u^{2H-2}\,du\,ds\\
    &=\frac1{a+b}\int_0^\infty(e^{-au}+e^{-bu})u^{2H-2}\,du\\
    &=\frac{\Gamma(2H-1)}{a+b}(a^{1-2H}+b^{1-2H}).
}
Taking $a=A_i$, $b=A_j$ and using $(2H-1)\Gamma(2H-1)=\Gamma(2H)$ proves \eqref{eq:joint-long-memory-covariance}. On the diagonal it gives $\Sigma^H_{ii}=\sigma_H^2H\Gamma(2H)A_i^{2-2H}$, agreeing with Theorem~\ref{thm:long-memory-last-iterate}.

\paragraph{Step 2: Rank of the covariance.}
The representation \eqref{eq:fBM} can be written using the same two-sided Brownian motion $\mathbb B$ as
\eqs*{B_H(s)=\int_{\mathbb R}\left(\int_0^s(r-u)_+^{-\beta}\,dr\right)d\mathbb B(u),}
where $(r-u)_+^{-\beta}$ denotes $(r-u)^{-\beta}\1(r>u)$. Integration by parts gives $Y_i=A_i^2\int_0^\infty e^{-A_i s}B_H(s)\,ds$. Since
\eqs*{A_i^2\int_0^\infty e^{-A_i s}\|B_H(s)\|_2\,ds=\sigma_H A_i^2\int_0^\infty e^{-A_i s}s^H\,ds<\infty,}
stochastic Fubini applies to this representation. Interchanging the nonnegative deterministic integrals then gives
\eqs*{
    Y_i&=\int_{\mathbb R}\varphi_i(u)\,d\mathbb B(u),\\
    \varphi_i(u)&=A_i^2\int_0^\infty e^{-A_i s}\int_0^s(r-u)_+^{-\beta}\,dr\,ds\\
    &=A_i\int_{\max(0,u)}^\infty e^{-A_i r}(r-u)^{-\beta}\,dr.
}
For $u>0$, substitution $x=r-u$ yields
\eqs*{\varphi_i(u)=A_i e^{-A_i u}\int_0^\infty e^{-A_i x}x^{-\beta}\,dx=\Gamma(1-\beta)A_i^\beta e^{-A_i u}.}
For $-1\le u\le0$, $|\varphi_i(u)|\le\Gamma(1-\beta)A_i^\beta$, and for $u<-1$, $|\varphi_i(u)|\le|u|^{-\beta}$. Together with $2\beta>1$, these bounds also verify $\varphi_i\in L^2(\mathbb R)$.

Suppose first that the population curvatures are distinct. For any $d_1,\ldots,d_k\in\mathbb R$, Brownian isometry gives
\eqs*{\Var\left(\sum_{i=1}^k d_iY_i\right)=\int_{\mathbb R}\left(\sum_{i=1}^k d_i\varphi_i(u)\right)^2du.}
If this variance is zero, the continuous function $\sum_i d_i\Gamma(1-\beta)A_i^\beta e^{-A_i u}$ vanishes for every $u>0$. Differentiating at a fixed $u_0>0$ gives
\eqs*{\sum_{i=1}^k d_i\Gamma(1-\beta)A_i^\beta e^{-A_i u_0}(-A_i)^r=0,\qquad r=0,\ldots,k-1.}
The Vandermonde matrix $((-A_i)^r)_{0\le r\le k-1,\,1\le i\le k}$ is invertible when the $A_i$ are distinct. Hence every $d_i=0$, proving positive definiteness. Equal population curvatures give the same function $\varphi_i$ and the same $Y_i$. Grouping equal population curvatures leaves one linearly independent coordinate per distinct population curvature, proving the rank assertion.
\end{proof}

\subsubsection{Comparison with batch expectiles}
\label{app:joint-batch-comparison}

We prove \eqref{eq:joint-batch-reduction} and the joint batch limit stated in Section~\ref{subsec:joint-long-memory}.

\begin{proof}
Put $\Delta_{m,i}:=\textcolor{black}{\hat\theta}_{m,i}^{\rm batch}-\theta_i^*$. The empirical score is strongly increasing with slope at least $\underline\tau_i$, so its root satisfies
\eqs*{\underline\tau_i|\Delta_{m,i}|\le\left|\frac1m\sum_{t=1}^m\xi_t^{(i)}\right|=O_P(a_m).}
Consequently $\Delta_{m,i}=O_P(a_m)=o_P(1)$, since $H<1$. Expanding the empirical equation at $\theta_i^*$ gives
\eqs*{0=\frac1m\sum_{t=1}^m\xi_t^{(i)}+\Delta_{m,i}\frac1m\sum_{t=1}^m a_t^{(i)}+\bar R_{m,i},}
where the kink bound implies
\eqs*{\frac{|\bar R_{m,i}|}{a_m}\le |1-2\tau_i|\frac{|\Delta_{m,i}|}{a_m}\frac1m\sum_{t=1}^m\1(|X_t-\theta_i^*|\le|\Delta_{m,i}|).}
To control the empirical fraction, fix $\eta,\epsilon>0$. Then
\eqs*{
    \Prob\left(\frac1m\sum_{t=1}^m\1(|X_t-\theta_i^*|\le|\Delta_{m,i}|)>\epsilon\right)
    &\le\Prob(|\Delta_{m,i}|>\eta)
    +\Prob\left(\frac1m\sum_{t=1}^m\1(|X_t-\theta_i^*|\le\eta)>\epsilon\right).
}
The first term tends to zero. By ergodicity, the empirical average in the second term converges almost surely to $\Prob(|X_0-\theta_i^*|\le\eta)$. The innovation density and $c_0\ne0$ imply that $X_0$ has no atom at $\theta_i^*$, so this probability tends to zero as $\eta\downarrow0$. Choosing $\eta$ with this probability below $\epsilon$ proves that the empirical fraction with radius $|\Delta_{m,i}|$ is $o_P(1)$. Therefore $\bar R_{m,i}=o_P(a_m)$.

Apply \eqref{eq:joint-score-fclt} after including the level $\tau=1/2$, if necessary. At that level $\theta_{1/2}^*=0$, $g_{1/2}(-X_t)=-X_t/2$, and $A=1/2$. With $v(m)=m^H\ell(m)=ma_m$, the joint limit at $s=1$ gives
\eqs*{\left(\frac1{A_i v(m)}\sum_{t=1}^m\xi_t^{(i)},-\frac1{v(m)}\sum_{t=1}^mX_t\right)\Rightarrow(-B_H(1),-B_H(1)).}
The difference of the two coordinates converges in probability to zero. Therefore,
\eqs*{\frac1{mA_i}\sum_{t=1}^m(\xi_t^{(i)}+A_iX_t)=a_m\left(\frac1{A_i v(m)}\sum_{t=1}^m\xi_t^{(i)}+\frac1{v(m)}\sum_{t=1}^mX_t\right)=o_P(a_m).}
Also, $m^{-1}\sum_{t=1}^m a_t^{(i)}\asto A_i>0$. Solving the empirical equation gives
\eqs*{
    \Delta_{m,i}
    &=-\frac{m^{-1}\sum_{t=1}^m\xi_t^{(i)}+\bar R_{m,i}}{m^{-1}\sum_{t=1}^m a_t^{(i)}}\\
    &=-\frac1{mA_i}\sum_{t=1}^m\xi_t^{(i)}+o_P(a_m)\\
    &=\bar X_m-\frac1{mA_i}\sum_{t=1}^m(\xi_t^{(i)}+A_iX_t)+o_P(a_m)\\
    &=\bar X_m+o_P(a_m).
}
This holds jointly for finitely many levels, proving \eqref{eq:joint-batch-reduction} and the joint limit.
\end{proof}

\section{Supplementary Numerical Experiments}\label{app:supplementary-numerical}

This section gives the simulation settings, calibration procedures, and additional numerical summaries for Section~\ref{sec:experiment}. We use the four observation models defined there.

\subsection{Simulation setup}\label{app:numerical-setup}

The AR(1) process starts from its stationary $\N(0,4/3)$ distribution. For GARCH$(1,1)$ and the nonlinear AR process, we discard 2,000 preliminary observations before starting SGD. The ARFIMA$(0,0.25,0)$ observations are generated as stationary Gaussian paths by Davies--Harte circulant embedding of the autocovariance \citep{DaviesHarte1987Tests}. This construction uses the covariance of the infinite moving average directly and requires neither truncation of its coefficients nor a preliminary observation burn-in.

For each model, we use $R=30{,}000$ independent replicates and constant learning rates
\[
\gamma\in\{0.1,\,0.05,\,0.01,\,10^{-3},\,10^{-4}\}.
\]
Within each replicate, all learning rates use the same observations. The marginal experiments use $\tau=0.5,0.9$, with $\tau=0.9$ displayed in Figure~\ref{fig:kolmogorov}(a)--(d). The joint experiment runs the levels $\tau=0.1,0.5,0.9$ on each common ARFIMA path, and Figure~\ref{fig:joint-contrast}(e)--(f) displays the pair $(0.5,0.9)$.

Each recursion starts at $\theta_0=\theta_\tau^*$, using the calibrated target described below when it is not available analytically. All learning rates use the same number of SGD updates,
\[
n=\left\lceil\frac{2\log(1/\gamma_{\min})}{\gamma_{\min}\underline\tau_{\min}}\right\rceil=1{,}842{,}069,
\qquad \gamma_{\min}=10^{-4},\qquad \underline\tau_{\min}:=\min_i\min(\tau_i,1-\tau_i)=0.1.
\]
This choice satisfies the sufficient burn-in conditions of Theorems~\ref{thm:short-memory-main} and~\ref{thm:long-memory-last-iterate} throughout the learning-rate grid. The preliminary observation burn-in is separate from these $n$ SGD updates.

\subsection{Calibration of targets and normalization constants}\label{app:numerical-calibration}

For AR(1) and ARFIMA, the stationary marginal is Gaussian. Let $z_\tau$ be the standard normal expectile, and let $\Phi$ and $\varphi$ denote the standard normal distribution function and density. We use
\[
\theta_\tau^*=\sigma_X z_\tau,\qquad A=\tau+(1-2\tau)\Phi(z_\tau),
\]
with
\[
\sigma_X^2=\frac{1}{1-0.5^2}=\frac43\quad\text{for AR(1)},\qquad
\sigma_X^2=\frac{\Gamma(1-2d)}{\Gamma(1-d)^2}\quad\text{for ARFIMA},\qquad d=0.25.
\]
For AR(1), the score long-run variance $\Omega_g$ is available analytically at $\tau=0.5$. At $\tau=0.9$, it is estimated by averaging batch-means estimates from six independent paths of length $N=6\times10^6$, using batch length $\lfloor\sqrt N\rfloor$.

For GARCH and the nonlinear AR process, calibration uses 256 independent paths of length 800,000, separately from the SGD experiment. Conditional on $\F_{t-1}$, write $X_t\sim\N(\mu_t,\sigma_t^2)$. Here $\mu_t=0$ for GARCH, and $(\mu_t,\sigma_t)=(0.9|X_{t-1}|,1)$ for the nonlinear AR process. With $z_t(\theta)=(\theta-\mu_t)/\sigma_t$, the conditional score has the explicit form
\[
h_t(\theta):=\E[g_\tau(\theta-X_t)\mid\F_{t-1}]
=\bigl(\tau+(1-2\tau)\Phi(z_t(\theta))\bigr)(\theta-\mu_t)+(1-2\tau)\sigma_t\varphi(z_t(\theta)).
\]
We obtain the target by setting the average of $h_t(\theta)$ over paths and times to zero, and estimate $A$ by averaging $\tau+(1-2\tau)\Phi(z_t(\theta))$ at that target. The delta-method standard error of the target is the sample standard deviation of the 256 path averages of $h_t$, evaluated at the calibrated target, divided by $\sqrt{256}A$. We estimate $\Omega_g$ by averaging pathwise batch-means estimates of the observed scores with batch length 894. For GARCH at $\tau=0.5$, we use the analytic target and score long-run variance.

Let $s(\gamma)$ denote the normalization for the marginal Gaussian approximation. For the three short-memory models,
\[
s(\gamma)^2=\gamma V,\qquad V=\frac{\Omega_g}{2A}.
\]
For ARFIMA, $\ell\equiv1/\Gamma(d)$ and $H=d+1/2$, so Theorem~\ref{thm:long-memory-last-iterate} gives
\[
s(\gamma)^2=V\gamma^{1-2d},\qquad V=\Gamma(2d)A^{1-2d}c_X,\qquad c_X=\frac{\Gamma(1-2d)}{\Gamma(d)\Gamma(1-d)}.
\]
The ARFIMA normalization therefore uses analytic constants and does not require a score long-run variance estimate. Table~\ref{tab:numerical-calibration} reports the calibration values for the two levels used in the marginal summaries.

\begin{table}[h]
\caption{Targets and normalization constants. The s.e. column gives the Monte Carlo standard error of the calibrated target, with a dash for analytically determined targets. For short memory, $V=\Omega_g/(2A)$. For ARFIMA, $V$ is the coefficient of $\gamma^{1-2d}$ in $s(\gamma)^2$, and the score long-run variance is infinite.}
\label{tab:numerical-calibration}
\begin{center}
\begin{tabular}{lrrrrrr}
\textbf{Model} & \boldmath$\tau$ & \boldmath$\theta_\tau^*$ & \textbf{s.e.} & \boldmath$A$ & \boldmath$\Omega_g$ & \boldmath$V$ \\
\hline \\
AR(1) & 0.5 & 0 & --- & 0.50000 & 1.00000 & 1.00000 \\
AR(1) & 0.9 & 0.99488 & --- & 0.25556 & 0.32527 & 0.63638 \\
GARCH$(1,1)$ & 0.5 & 0 & --- & 0.50000 & 0.25000 & 0.25000 \\
GARCH$(1,1)$ & 0.9 & 0.85814 & $4.5\times10^{-5}$ & 0.25153 & 0.17481 & 0.34750 \\
Nonlinear AR & 0.5 & 1.64726 & $2.7\times10^{-4}$ & 0.50000 & 5.13856 & 5.13856 \\
Nonlinear AR & 0.9 & 3.11902 & $4.3\times10^{-4}$ & 0.23916 & 2.93354 & 6.13301 \\
ARFIMA$(0,0.25,0)$ & 0.5 & 0 & --- & 0.50000 & $\infty$ & 0.50000 \\
ARFIMA$(0,0.25,0)$ & 0.9 & 0.93606 & --- & 0.25556 & $\infty$ & 0.35747 \\
\end{tabular}
\end{center}
\end{table}

\subsection{Gaussian approximation and mean squared error}\label{app:numerical-marginal}

At each checkpoint, we compute the Kolmogorov distance used in Section~\ref{sec:experiment} from the $R$ standardized iterates. The 93 checkpoints are obtained by rounding 100 logarithmically spaced times between 1 and $n$ and removing duplicates. The shaded region in Figure~\ref{fig:kolmogorov}(a)--(d) extends up to the 95th percentile of the same statistic for $R=30{,}000$ independent standard normal draws, approximately $0.0078$. The reference is pointwise in time and excludes uncertainty in calibrated targets and normalization constants.

We also compute the empirical mean squared error of the last iterate,
\[
\widehat{\operatorname{MSE}}_\gamma=\frac1R\sum_{r=1}^R(\theta_n^{(r)}-\theta_\tau^*)^2.
\]
Table~\ref{tab:numerical-mse} reports its ratio to $s(\gamma)^2$. The fitted slope is the ordinary least-squares slope of $\log\widehat{\operatorname{MSE}}_\gamma$ on $\log\gamma$ over all five learning rates.

\begin{table}[h]
\caption{Empirical MSE divided by $s(\gamma)^2$ at the five learning rates, and fitted log--log slopes. Each MSE uses $R=30{,}000$ replicates at $n=1{,}842{,}069$. The final column gives the exponent of the theoretical variance scale.}
\label{tab:numerical-mse}
\begin{center}
\begin{tabular}{lrrrrrrrr}
& & \multicolumn{5}{c}{\boldmath$\widehat{\operatorname{MSE}}_\gamma/s(\gamma)^2$} & \multicolumn{2}{c}{\textbf{Slope}} \\
\textbf{Model} & \boldmath$\tau$ & \boldmath$0.1$ & \boldmath$0.05$ & \boldmath$0.01$ & \boldmath$10^{-3}$ & \boldmath$10^{-4}$ & \textbf{Fitted} & \textbf{Theory} \\
\hline \\
AR(1) & 0.5 & 0.951 & 0.965 & 0.981 & 0.993 & 0.996 & 0.994 & 1.0 \\
AR(1) & 0.9 & 0.997 & 1.012 & 1.011 & 1.017 & 1.022 & 0.997 & 1.0 \\
GARCH$(1,1)$ & 0.5 & 1.020 & 1.014 & 1.013 & 1.008 & 1.000 & 1.003 & 1.0 \\
GARCH$(1,1)$ & 0.9 & 0.899 & 0.945 & 0.994 & 0.998 & 1.013 & 0.985 & 1.0 \\
Nonlinear AR & 0.5 & 0.837 & 0.907 & 0.983 & 0.998 & 1.003 & 0.977 & 1.0 \\
Nonlinear AR & 0.9 & 0.736 & 0.845 & 0.966 & 1.019 & 1.013 & 0.958 & 1.0 \\
ARFIMA$(0,0.25,0)$ & 0.5 & 1.005 & 1.003 & 1.005 & 0.998 & 1.005 & 0.500 & 0.5 \\
ARFIMA$(0,0.25,0)$ & 0.9 & 1.077 & 1.057 & 1.027 & 1.014 & 0.992 & 0.511 & 0.5 \\
\end{tabular}
\end{center}
\end{table}

At $\gamma=10^{-4}$ and $\tau=0.9$, the nonlinear AR process has an MSE ratio of $1.013$, but its Kolmogorov distance $0.0151$ remains above the Gaussian reference threshold $0.0078$. Thus the second-moment comparison and the distributional comparison measure different aspects of the approximation.

For nonlinear AR at $\tau=0.9$ and $\gamma=10^{-4}$, a target error of one reported standard error gives a standardized shift $\delta=(4.3\times10^{-4})/\sqrt{10^{-4}\cdot6.13301}\simeq0.0174$. Even for a Gaussian shape, this yields $d_K(\N(\delta,1),\N(0,1))=2\Phi(|\delta|/2)-1\simeq0.0069$. This sensitivity calculation does not estimate the realized calibration error, but shows that calibration uncertainty may contribute to the discrepancy.

\subsection{Joint errors and comparison with batch expectiles}\label{app:numerical-joint}

We focus on $(\tau_i,\tau_j)=(0.5,0.9)$ because their population curvatures differ, with $(A_i,A_j)\approx(0.5,0.25556)$. Corollary~\ref{cor:joint-long-memory-covariance} therefore predicts a nondegenerate limiting SGD contrast for this pair.

We compute the batch expectiles of Section~\ref{subsec:joint-long-memory} from the last $m=\lceil1/\gamma\rceil$ observations of each path, and also use the longer windows $m=10^5,10^6$. The latter two windows have no corresponding SGD run. The normalizations are
\[
s_\gamma=\frac{\gamma^{1/2-d}}{\Gamma(d)},\qquad a_m=\frac{m^{d-1/2}}{\Gamma(d)},\qquad d=0.25.
\]
At level $\tau_i$, the marginal normalization satisfies $s(\gamma)^2=s_\gamma^2\Sigma^H_{ii}$ and includes the level-specific Gaussian variance, whereas $s_\gamma$ is common to all levels. For the matched windows $m=1/\gamma$, we have $a_m=s_\gamma$. The SGD iterate uses the full path through time $n$, while the batch estimate uses its final $m$ observations.

Figure~\ref{fig:joint-contrast}(e) displays the first 4,000 normalized error pairs among the 30,000 replicates at $\gamma=10^{-4}$. Its 50\% and 95\% contours are those of the theoretical centered Gaussian limit, with covariance from \eqref{eq:joint-long-memory-covariance},
\[
\Sigma^H\approx\begin{pmatrix}6.5725&5.3327\\5.3327&4.6989\end{pmatrix}.
\]
For Figure~\ref{fig:joint-contrast}(f), we compute the normalized centered contrasts
\begin{align*}
C_\gamma^{\rm SGD}&=\frac{(\theta_n^{(j)}-\theta_n^{(i)})-(\theta_j^*-\theta_i^*)}{s_\gamma},\\
C_m^{\rm batch}&=\frac{({\hat\theta}_{m,j}^{\rm batch}-{\hat\theta}_{m,i}^{\rm batch})-(\theta_j^*-\theta_i^*)}{a_m}.
\end{align*}
Each plotted width is the empirical 95th percentile minus the empirical 5th percentile of the relevant contrast across all 30,000 replicates. We obtain its error bar from the 2.5th and 97.5th percentiles of the widths in 300 bootstrap resamples of the replicates. Since all learning rates and batch windows use the same paths, the points along a curve are dependent. The error bars are pointwise 95\% bootstrap intervals for the widths.

The limiting SGD contrast variance and its central 90\% interval width are
\[
v_C=\Sigma^H_{ii}+\Sigma^H_{jj}-2\Sigma^H_{ij}\approx0.6061,\qquad
2\Phi^{-1}(0.95)\sqrt{v_C}\approx2.561.
\]
At $\gamma=10^{-4}$, the empirical SGD width is $2.591$, close to the limit $2.561$, while the batch width decreases from $1.283$ at $m=10^4$ to $0.475$ at $m=10^6$. These values support the positive limiting SGD contrast variance and the cancellation of the leading batch errors described in Section~\ref{subsec:joint-long-memory}.

\end{document}